\documentclass{article}
\usepackage[utf8]{inputenc}
\usepackage[T1]{fontenc}
\usepackage{times}
\usepackage{helvet}
\usepackage{courier}

\usepackage{amsmath}
\usepackage{amssymb}
\usepackage{amsfonts}

\usepackage{booktabs}
\usepackage{multirow}
\usepackage{makecell}
\usepackage{array}
\usepackage{tabularx}

\usepackage{graphicx}
\usepackage{xcolor}
\usepackage{caption}

\usepackage[linesnumbered,ruled,vlined]{algorithm2e}

\usepackage{enumitem}
\usepackage{tcolorbox}
\usepackage{float}
\usepackage{nicefrac}
\usepackage{microtype}
\usepackage{mathtools}
\DeclareMathOperator{\diag}{diag}
\usepackage{url}
\usepackage[colorlinks=true, linkcolor=blue, citecolor=blue, urlcolor=blue]{hyperref}

\usepackage{tikz}

\usepackage[numbers,square]{natbib}

\usepackage{arxiv}

\usepackage{amsthm}

\newtheorem{theorem}{Theorem}[section]
\newtheorem{lemma}[theorem]{Lemma}
\newtheorem{proposition}[theorem]{Proposition}

\title{ModalImmune: Immunity Driven Unlearning via Self Destructive Training}

\author{
    Rong Fu$^{\dagger}$\thanks{Corresponding author: \texttt{mc46603@um.edu.mo}} \\
    University of Macau \\
    \texttt{mc46603@um.edu.mo} \\
    \And
    WeiZhi Tang$^{\dagger}$ \\
    University of Macau \\
    \texttt{mc56666@um.edu.mo} \\
    \And
    Ziming Wang$^{\dagger}$ \\
    Zhejiang University \\
    \texttt{w986827512@zuaa.zju.edu.cn} \\
    \And
    Jia Yee Tan\\
    Renmin University of China \\
    \texttt{tanjiayi2002@ruc.edu.cn} \\
    \And
    Zijian Zhang\\
    University of Pennsylvania \\
    \texttt{zzjharry@alumni.upenn.edu} \\
    \And
    Zhaolu Kang\\
    Peking University \\
    \texttt{kangzl9966@stu.pku.edu.cn} \\
    \And
    Muge Qi\\
    Peking University \\
    \texttt{2301210659@stu.pku.edu.cn} \\
    \And
    Shuning Zhang\\
    Tsinghua University \\
    \texttt{zsn23@mails.tsinghua.edu.cn} \\
    \And
    Simon Fong\\
    University of Macau \\
    \texttt{ccfong@um.edu.mo} \\
    \AND
    \normalfont{$^\dagger$These authors contributed equally to this work.}
}
\renewcommand{\headeright}{}
\renewcommand{\undertitle}{}
\renewcommand{\shorttitle}{ModalImmune}

\hypersetup{
    pdftitle={ModalImmune: Immunity Driven Unlearning via Self Destructive Training},
    pdfsubject={cs.MM, cs.LG, cs.CL}, 
   pdfauthor={Rong Fu, WeiZhi Tang, Ziming Wang, Jia Yee Tan, Zijian Zhang, Zhaolu Kang, Muge Qi, Shuning Zhang, Simon Fong},
    pdfkeywords={multimodal robustness, immunity-driven unlearning, self-destructive training, spectral collapse, hyper-gradient adaptation},
    pdfstartview={FitH},
    colorlinks=true,
    linkcolor=red,
    citecolor=green,
    filecolor=magenta,
    urlcolor=cyan
}

\begin{document}
\maketitle

\begin{abstract}
Multimodal systems are vulnerable to partial or complete loss of input channels at deployment, which undermines reliability in real-world settings. This paper presents ModalImmune, a training framework that enforces modality immunity by intentionally and controllably collapsing selected modality information during training so the model learns joint representations that are robust to destructive modality influence. The framework combines a spectrum-adaptive collapse regularizer, an information-gain guided controller for targeted interventions, curvature-aware gradient masking to stabilize destructive updates, and a certified Neumann-truncated hyper-gradient procedure for automatic meta-parameter adaptation. Empirical evaluation on standard multimodal benchmarks demonstrates that ModalImmune improves resilience to modality removal and corruption while retaining convergence stability and reconstruction capacity.
\end{abstract}


\section{Introduction}
Multimodal learning integrates complementary signals from heterogeneous channels such as text, audio and vision to achieve higher prediction accuracy and richer inference. Despite advances, models that assume complete and well-aligned inputs often fail in realistic deployments where one or more modalities may be missing, corrupted, or unavailable due to sensor faults, privacy constraints, communication dropouts, or adversarial manipulation. This fragility hinders adoption of multimodal systems in safety-critical and resource-constrained scenarios and motivates methods that improve robustness to missing or destructive modalities \cite{wu2024deep, zhang2024multimodal, reza2024robust}.

Existing remedies follow several directions. Generative imputation reconstructs missing channels from available inputs using conditional generators or diffusion models, which can restore performance in some cases but risks hallucination and incurs additional computational cost \cite{zhang2024multimodal, wu2024deep}. Architectural strategies build modular fusion, expert routing or adaptive gating to reduce over-reliance on any single modality, yet many such solutions are trained for specific missing patterns or require extra components that constrain generality \cite{li2025simmlm, reza2024robust}. Information-centric objectives and contrastive schemes aim to shape modality-invariant embeddings and to suppress spurious unimodal cues, but they do not explicitly enforce immunity against a modality that may become destructive when present \cite{liu2024contrastive, cui2024enhancing, guo2024multimodal}.

Three limitations persist. First, imputational approaches may introduce low-fidelity signals and impose runtime burdens that limit practicality. Second, many adaptive fusion or reconstruction techniques assume fixed missing patterns or add modules that hinder architectural flexibility. Third, there is a lack of principled training protocols that deliberately expose models to controlled destructive interventions so that joint representations become inherently resistant to a modality’s harmful influence while preserving task-relevant information from remaining channels.

To address these gaps we introduce ModalImmune, a unified training protocol that operationalizes immunity-driven unlearning via Self Destructive Training. The protocol repeatedly and selectively enforces a controlled collapse on a chosen modality during training so that the encoder and fusion pathways learn to down-weight or ignore destructive information. The collapse is spectrum-adaptive to remove informative directions without destabilizing optimization. An information-gain guided controller prioritizes high-impact interventions to improve training efficiency. Curvature-aware gradient masking prevents unstable gradient ascent when destructive updates are applied. Meta-parameters that govern collapse strength and stabilization are adapted by a certified Neumann-truncated hyper-gradient procedure that provides stable bi-level optimization.

Our contributions are as follows. First, we propose Self Destructive Training as a new paradigm for enforcing modality immunity by targeted, controlled information collapse during model training. Second, we design a spectrum-adaptive collapse regularizer together with an information-gain driven controller that identifies and prioritizes high-impact modality interventions. Third, we develop curvature-aware gradient masking to stabilize destructive updates and a certified Neumann-truncated hyper-gradient algorithm to adapt meta-parameters automatically. Fourth, we provide comprehensive experiments showing that ModalImmune yields representations that maintain predictive performance under modality removal and corruption while preserving convergence behavior and reconstruction capability.

\section{Related Work}
\label{sec:related}

\subsection{Learning with missing modalities}
Robust handling of incomplete modality sets is a central concern for contemporary multimodal learning. A broad family of approaches focuses on reconstructing absent channels or learning modality-agnostic representations so that downstream predictors remain effective when some inputs are unavailable. Recent works propose distillation and translation schemes that reconstruct missing semantics from available modalities and align representations across modalities to reduce the performance gap caused by missing data \cite{li2024correlation, li2024toward, luo2023multimodal}. The correlation-decoupled knowledge distillation framework and hierarchical representation learning both demonstrate that transferring structured cross-sample and category-level knowledge improves robustness under uncertain missingness. 

In parallel, medical-imaging and healthcare applications highlight practical constraints where modality absence is common. Masking-based reconstruction ideas that treat missing inputs as masked features and learn to reconstruct them via feature-to-feature mapping have been shown effective for clinical segmentation tasks \cite{zeng2024missing}. These methods demonstrate the value of exposing models during training to realistic missing patterns so they generalize to heterogeneous deployment scenarios.

\subsection{Generative imputation and joint latent models}
Generative latent-variable models and variational autoencoder families have been widely applied to multimodal imputation and embedding. Joint VAE formulations offer a principled route to infer missing channels while preserving modality-specific structure, and they have been validated on complex biological and sensing datasets \cite{cohen2023joint}. Related lines of work exploit conditional generative models and cross-modal synthesis to produce surrogate inputs that feed into fusion modules; such generative imputation often yields better downstream performance than naive zero- or mean-filling \cite{zhang2024unified, zhang2025generative}.

\subsection{Contrastive learning and representation regularization}
Contrastive objectives and prototype-guided alignment strategies are now standard tools to obtain discriminative multimodal embeddings. Methods that combine sample-level contrastive signals with prototype or curriculum mechanisms improve class separation and resilience in low-data settings \cite{sun2023progressive, mai2022multimodal}. Distillation that preserves inter-sample correlations has also been used to reconstruct missing semantics and to regularize student representations for robustness \cite{li2024correlation}. These results point to a recurring design pattern: learning objectives that explicitly shape inter- and intra-modal geometry confer robustness under partial observability. 

\subsection{Architectural strategies for partial observability}
Architectures tailored for partial modality availability include modular fusion, expert routing and reconstruction-augmented fusion. Graph and hypergraph constructions capture higher-order inter-modal relations and can mitigate the effects of modality noise and absence \cite{zhang2022m3care, zeng2024missing}. Transformer-based translation modules that map non-text modalities into richer textual spaces also yield robust joint features by leveraging pretrained textual priors \cite{liu2024modality, guo2024multimodal}. Empirical comparisons underscore that combining structured fusion with learned imputation produces more graceful degradation than monolithic fusion designs. 

\subsection{Information-theoretic and bottleneck regularizers}
Imposing information constraints on representations is an effective route to discard nuisance variability while preserving task-relevant signals. The multimodal information bottleneck and related variational formulations reduce redundancy and suppress noisy unimodal components, thereby improving generalization and robustness \cite{mai2022multimodal, cui2024enhancing}. These principles are complementary to generative imputation and can be integrated as regularizers to encourage compact, task-focused embeddings.

\subsection{Representation collapse and mitigation}
Recent analyses have examined the phenomenon of modality collapse, where fusion networks disproportionately rely on a subset of channels and ignore others. Theoretical and empirical studies show that collapse can arise from entanglement in the fusion head or rank bottlenecks in student encoders, and that targeted basis reallocation or cross-modal regularization can alleviate the problem \cite{chaudhuri2025closer}. These insights motivate design choices that explicitly preserve representational capacity per modality and that actively redistribute information when collapse tendencies are detected. 

\subsection{Optimization, bilevel adaptation and gradient control}
Hyperparameter adaptation through bilevel optimization and approximate hyper-gradient estimation has been used to tune meta-parameters that control reconstruction strength and regularization intensity \cite{jie2022adaptive, huang2021biadam, zhang2023data}. Practical algorithms exploit truncated Neumann-series approximations to stabilize hyper-gradient computation while keeping compute overhead manageable. Complementary work on gradient masking and routing demonstrates how localized gradient control can limit destructive interference across modules and enable targeted feedback mechanisms \cite{cloud2024gradient, kaul2022projective}. 

\subsection{Bandit selection, information gain and robust evaluation}
Adaptive selection of interventions and targeted perturbations benefits from information-aware bandit strategies and causal regret analyses. Information-gain driven arm selection offers a principled mechanism to identify which modality to perturb or mask during training so that the model learns robust conditional dependencies \cite{vakili2021information, yan2024causal, sukhija2024maxinforl}. For evaluation, recent benchmarks and corruption studies systematically probe model robustness under varied missingness and noise regimes, proposing standardized metrics and stress tests to quantify practical resilience \cite{liao2025benchmarking, das2023multimodal}. 

\subsection{Positioning of this work}
This study introduces a unified training strategy that integrates spectrum-adaptive collapse, curvature-aware gradient regulation, and certified hyper-gradient tuning to achieve modality immunity. Unlike imputation or single-pass distillation, the proposed protocol deliberately applies controlled destructive interventions while retaining reconstruction capability. Compared to Random Modality Dropout, which injects passive noise, our approach actively removes high-impact spectral components, creating a calibrated information deficit. Unlike adversarial training that optimizes worst-case perturbations under norm constraints, our method employs an information-theoretic controller to prioritize impactful collapses and stabilizes updates through curvature gating. These design choices distinguish the framework from existing robustness schemes and establish a principled pathway toward adaptive multimodal resilience.

\begin{figure*}[t]
  \centering
  \includegraphics[width=0.98\linewidth]{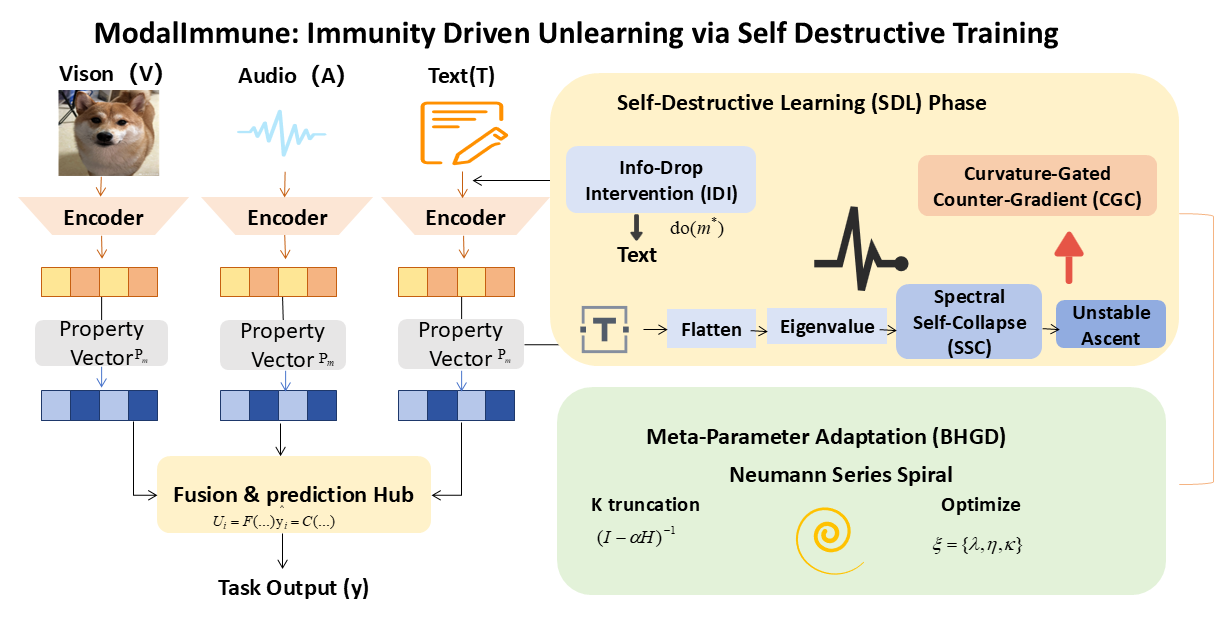}
  \caption{Overview of the \textbf{ModalImmune} framework, which treats modality destruction as an active \textbf{causal intervention}. 
  The training strategy alternates between standard reconstruction and \textbf{Self-Destructive Learning (SDL)} built on three key components: 
  \textbf{Info-Drop Intervention (IDI)}, where an EXP3.P bandit controller leverages an information-gain surrogate $\ell_m$ to adaptively select the target modality $m^\star$; 
  \textbf{Spectral Self-Collapse (SSC)}, enforcing an irreversible directional loss through a spectrum-adaptive regularizer $L_{\mathrm{coll}}$ combined with a stable-rank penalty; 
  \textbf{Curvature-Gated Counter-Gradient (CGC)}, which inspects empirical Fisher information to apply a masking multiplier $\rho$ that mitigates destabilizing gradient ascent. 
  Meanwhile, meta-parameters $\xi = \{\lambda,\eta,\kappa\}$ are optimized via \textbf{Bi-level Hyper-Gradient Descent (BHGD)} using a certified \textbf{Neumann-truncated} estimator. 
  This unified protocol ensures the \textbf{Fusion Hub} and \textbf{Task Output} $(C)$ remain robust under counterfactual interventions and adversarial modality collapse.}
  \label{fig:modalimmune_architecture}
\end{figure*}

\section{Methodology}
\label{sec:method}

We introduce \textbf{ModalImmune}, a unified training protocol that promotes modality immunity by deliberately and controllably destroying task-relevant information in a selected modality during training. The protocol operates by alternating standard reconstruction-driven updates with targeted \emph{Self-Destructive Learning} (SDL) phases that enforce spectrum-adaptive collapse on one modality at a time, guided by an information-gain controller and stabilized with curvature-aware gradient gating. Meta-parameters that regulate collapse strength and stabilization are adapted online via a certified Neumann-truncated hyper-gradient estimator.

\subsection{Self-Destructive Learning as a causal intervention}
\label{sec:causal-view}

We view missing or destroyed modalities as active causal interventions rather than passive noise. Let the complete set of modality embeddings be $\{z_m\}_{m\in\mathcal{M}}$ and suppose the label $y$ is generated by a (possibly unknown) structural mapping
\begin{equation}
y \leftarrow g\big(\{z_m\}_{m\in\mathcal{M}}\big),
\end{equation}
where $g(\cdot)$ denotes the label-generating mechanism. When we intervene on a modality $m^\star$ we apply a do-style operation that replaces $z_{m^\star}$ with a destroyed (collapsed) counterpart, producing the counterfactual
\begin{equation}
\tilde{y} \leftarrow g\big(\{z_m\}_{m\in\mathcal{M}\setminus\{m^\star\}},\; \mathrm{Collapse}(z_{m^\star})\big).
\end{equation}
where $\mathrm{Collapse}(\cdot)$ denotes a controlled information-destruction operator (defined below). Training under such interventions forces the fusion and prediction pathways to learn representations that are robust to the presence of damaged or adversarial modalities.

\subsection{Three design primitives}
\label{sec:primitives}

To operationalize the causal intervention above, ModalImmune relies on three low-level design primitives that are used throughout the subsequent equations and the training algorithm.

\paragraph{Info-Drop Intervention (IDI).}
An adaptive selector chooses the modality to intervene on in each SDL batch by estimating per-modality information gain and using a bandit controller to prioritize high-impact choices. This avoids wasting destructive steps on low-value targets and concentrates training on modalities that most affect task performance.

\paragraph{Spectral Self-Collapse (SSC).}
When a modality is chosen, its batch embedding matrix undergoes a controlled spectral collapse that removes dominant directions while preserving global scale. The collapse is stochastic and spectrum-adaptive so that it produces effectively irreversible loss of directional information rather than small perturbations.

\paragraph{Curvature-Gated Counter-Gradient (CGC).}
To prevent destructive updates from destabilizing optimization, a curvature gate inspects a Gauss–Newton or empirical Fisher approximation and either freezes the encoder gradients for the attacked modality or applies a bounded negative-feedback mask that prevents uncontrolled ascent.

\subsection{Notation}
Let $\mathcal{M}$ be the set of modalities and $\mathcal{D}=\{(x_i^m)_{m\in\mathcal{M}},y_i\}_{i=1}^N$ the labeled dataset of $N$ examples. Each modality $m$ has an encoder $f_m(\cdot;\theta_m)$ producing embeddings $z_i^m\in\mathbb{R}^d$. Fusion is performed by $u_i=F(\{z_i^m\}_{m\in\mathcal{S}_i};\phi)$ where $\mathcal{S}_i\subseteq\mathcal{M}$ denotes the modalities available (or synthesized) for sample $i$, and the task prediction is $\hat y_i = C(u_i;\psi)$. The supervised per-sample loss is denoted $L_{\mathrm{TASK}}$, the reconstruction loss by $L_{\mathrm{re}}$, the property-vector loss by $L_{\mathrm{pe}}$, and the contrastive alignment loss by $L_{\mathrm{con}}$.

\subsection{Model components and composite objective}
Each modality has a learnable property vector $p_m\in\mathbb{R}^d$ that captures sample-invariant traits. A conditional generator $\mathcal{G}_m(\cdot;\vartheta_m)$ can synthesize a surrogate embedding $\hat z_i^m$ for a missing modality by combining available embeddings with $p_m$. A decomposition module optionally maps $z_i^m$ to $(s_i^m,\mu_i^m)$ where $s_i^m$ is sample-specific variation and $\mu_i^m$ captures invariant characteristics. Back-translation networks $\mathcal{B}_m(\cdot;\beta_m)$ are used together with contrastive alignment to preserve modality-specific signals when needed.

SDL is implemented as a plug-in regularizer. The total training objective for a batch is
\begin{equation}
\mathcal{L}_{\mathrm{total}} = \mathcal{L}_{\mathrm{orig}} + \mathbb{I}_{\mathrm{SDL}}\;\lambda\;L_{\mathrm{coll}}(m^\star),
\end{equation}
where $\mathcal{L}_{\mathrm{orig}}$ aggregates the reconstruction and contrastive terms, $\mathbb{I}_{\mathrm{SDL}}\in\{0,1\}$ indicates whether the current batch is an SDL batch, $\lambda>0$ weighs the collapse regularizer, and $L_{\mathrm{coll}}(m^\star)$ is the spectral collapse penalty applied to the selected modality $m^\star$.
\subsection{Property Vector: Sample-Invariant Modality Signature}
\label{sec:property_vector}

The property vector $p_m$ is designed as a modality-specific signature shared across all samples, capturing global characteristics rather than per-utterance variations. Examples include the average spectral slope for acoustic streams or typical color distribution bias for visual encoders. During SDL, when the embedding of the targeted modality is collapsed, the conditional generator $G_m$ combines intact embeddings with this static signature to synthesize a plausible substitute. This ensures that the fusion hub receives inputs aligned with the original semantic scale. In essence, $p_m$ serves as a modality-level prior, whereas the conventional feature $z_i^m$ represents instantaneous, sample-dependent content.
\subsection{Training objectives}
Training alternates between standard reconstruction-driven updates and SDL updates. For a batch $B$ of size $n$ the reconstruction-driven objective reads
\begin{align}
\mathcal{L}_{\mathrm{rec}} &= \frac{1}{n}\sum_{i\in B} L_{\mathrm{TASK}}\big(C(F(\{z_i^m\}_{m\in\mathcal{S}_i};\phi);\psi), y_i\big) \\
&\quad + \alpha L_{\mathrm{re}} + \beta L_{\mathrm{pe}} + \gamma L_{\mathrm{con}},
\end{align}
where $\alpha,\beta,\gamma\ge0$ are scalar weights. Here $L_{\mathrm{TASK}}$ is cross-entropy for classification or mean-squared error for regression.

When SDL is active and modality $m^\star$ is selected, the SDL objective becomes
\begin{align}
\mathcal{L}_{\mathrm{SDL}} &= \frac{1}{n}\sum_{i\in B} L_{\mathrm{TASK}}\big(C(F(\{z_i^m\}_{m\ne m^\star};\phi);\psi), y_i\big) \\
&\quad + \gamma L_{\mathrm{con}} + \lambda L_{\mathrm{coll}}(m^\star),
\end{align}
which omits $m^\star$ from fusion and enforces the collapse penalty. The per-batch loss used for updates is the interpolation
\begin{equation}
\mathcal{L} = (1-\mathbb{I}_{\mathrm{SDL}})\,\mathcal{L}_{\mathrm{rec}} + \mathbb{I}_{\mathrm{SDL}}\,\mathcal{L}_{\mathrm{SDL}}.
\label{eq:total_loss}
\end{equation}
In all formulas $\mathbb{I}_{\mathrm{SDL}}$ denotes the SDL-mode indicator for the batch.

\subsection{Spectrum-adaptive collapse regularizer}
Let $\mathrm{Cov}_B$ be the empirical covariance of the selected-modality embeddings $\{z_i^{m^\star}\}_{i\in B}$ in the current batch. Define the per-batch spectral radius
\begin{equation}
\rho_B = \mathrm{spectral\_radius}\big(\{z_i^{m^\star}\}_{i\in B}\big),
\end{equation}
where the spectral radius returns the largest singular value of the batch matrix. We set a noise scale proportional to embedding magnitude:
\begin{equation}
\sigma_B = \frac{\rho_B}{\sqrt{d}},
\end{equation}
where $d$ is the embedding dimension. To obtain a rank-robust penalty when $n<d$ we use the stable-rank
\begin{equation}
\mathrm{stable\_rank}(\mathrm{Cov}_B) = \frac{\|\mathrm{Cov}_B\|_*}{\|\mathrm{Cov}_B\|_2},
\end{equation}
where $\|\cdot\|_*$ is the nuclear norm and $\|\cdot\|_2$ the operator norm. The collapse regularizer is
\begin{equation}
L_{\mathrm{coll}}(m^\star) = \frac{1}{n}\sum_{i\in B}\big\lVert z_i^{m^\star} - \varepsilon_i\big\rVert_2^2 + \eta\cdot\mathrm{stable\_rank}(\mathrm{Cov}_B),
\label{eq:stable_collapse}
\end{equation}
where $\varepsilon_i\sim\mathcal{N}(0,\sigma_B^2 I_d)$ and $\eta\ge0$ controls the stable-rank penalty. The stochastic perturbation plus stable-rank term yields an effectively irreversible reduction of directional information while keeping overall scale compatible with remaining modalities.

\subsection{Curvature-gated gradient masking}
Let $\mathcal{F}_{m^\star}$ denote the empirical Fisher (or Gauss–Newton) approximation for encoder $f_{m^\star}$ on the current batch and let $\lambda_{\min}$ be its smallest eigenvalue. Define a curvature threshold
\begin{equation}
\tau = \frac{0.01}{\mathrm{lr}}
\end{equation}
where $\mathrm{lr}$ is the encoder learning rate. The masking multiplier $\rho$ is chosen by
\begin{equation}
\rho =
\begin{cases}
0, & \lambda_{\min} < -\tau,\\[4pt]
-\kappa, & \lambda_{\min} \ge -\tau,
\end{cases}
\end{equation}
where $\kappa>0$ is the negative-feedback scale. When $\rho=0$ gradients for the attacked encoder are frozen; when $\rho=-\kappa$ a controlled negative feedback is applied to the encoder updates to avoid ascent-based divergence.

\subsection{Hyper-gradient estimation via certified Neumann truncation}
We treat meta-parameters $\xi=\{\lambda,\eta,\kappa\}$ as latent variables and approximate hyper-gradients of the validation loss $\mathcal{L}_{\mathrm{val}}(\xi)$ by
\begin{equation}
\nabla_{\xi}\mathcal{L}_{\mathrm{val}}
\approx \left.\frac{\partial \mathcal{L}_{\mathrm{val}}}{\partial \theta}\,(I-\alpha H)^{-1}\,\frac{\partial^2 \mathcal{L}_{\mathrm{train}}}{\partial \theta \partial \xi}\right|_{\theta^*(\xi)},
\label{eq:hypergrad}
\end{equation}
where $H=\partial^2\mathcal{L}_{\mathrm{train}}/\partial\theta^2$ is the Hessian approximation and $\alpha$ is the inner learning rate. The inverse $(I-\alpha H)^{-1}$ is approximated by a Neumann series truncated at depth $K$, and $K$ is chosen online via a doubling rule until the truncation residual is below a prescribed tolerance relative to stochastic estimation noise. This yields a verifiable truncation certificate and stabilizes BHGD updates.

\subsection{Information-gain surrogate for bandit selection}
To supply a causal-relevant reward to the EXP3.P bandit we approximate per-modality information gain on batch $B$ with the surrogate
\begin{equation}
\ell_m \approx \frac{1}{|B|}\sum_{i\in B}\Big( L_{\mathrm{task}}\big(\mathcal{M}\big) - L_{\mathrm{task}}\big(\mathcal{M}\setminus\{m\}\big)\Big)^2,
\label{eq:info_gain}
\end{equation}
where $L_{\mathrm{task}}(\mathcal{S})$ denotes the task loss computed when using modalities in the set $\mathcal{S}$. This proxy is dimensionless and approximates the conditional information $\mathcal{I}(z^m;y\mid z^{\mathcal{M}\setminus m})$ using current model predictions.

\subsection{Training algorithm (summary)}
\begin{algorithm}[h]
\caption{ModalImmune: INFO-GAIN bandit + BHGD}
\label{alg:modalimmune}
\KwIn{dataset $\mathcal{D}$, modalities $\mathcal{M}$, SDL probability $p_{\mathrm{sdl}}$, weights $\alpha,\beta,\gamma$, bandit and BHGD hyper-params}
Initialize EXP3.P weights for arms $m\in\mathcal{M}$ and meta-parameters $\xi$\;
\For{each epoch}{
  \For{each mini-batch $B\subset\mathcal{D}$}{
    sample $u\sim\mathrm{Uniform}(0,1)$\;
    \If{$u<p_{\mathrm{sdl}}$}{
      set $\mathbb{I}_{\mathrm{SDL}}\leftarrow 1$\;
      select arm $m^\star$ via EXP3.P using recent $\ell_m$ estimates\;
      set $\mathcal{S}_i \leftarrow \mathcal{M}\setminus\{m^\star\}$ for all $i\in B$\;
      forward-pass using modalities in $\mathcal{S}_i$, compute $L_{\mathrm{coll}}(m^\star)$ by Eq.~\eqref{eq:stable_collapse}\;
      compute gradients, apply curvature-gated masking and step the optimizer\;
      observe $\ell_{m^\star}$ and update EXP3.P\;
    }
    \Else{
      set $\mathbb{I}_{\mathrm{SDL}}\leftarrow 0$\;
      perform standard reconstruction-driven forward-pass (including generator synthesis for randomly masked modalities)\;
      update parameters by minimizing $\mathcal{L}_{\mathrm{rec}}$\;
    }
    periodically compute approximate hyper-gradient and update $\xi$ via BHGD\;
  }
}
\end{algorithm}

\subsection{Compatibility certificate and Lipschitz proxy}
To ease encoder replacement we provide a conservative spectral upper bound on the fusion Lipschitz constant for modality $m$:
\begin{equation}
L_F^{\langle m\rangle} \;=\; \sup_{x^m}\big\|\nabla_{x^m} F\circ f_{\theta_m}(x^m)\big\|_2
\le \|\mathbf{W}_F\|_2 \cdot \|\mathbf{W}_{\theta_m}\|_2,
\end{equation}
where $\mathbf{W}_F$ is the first linear layer (in the fusion network) receiving modality $m$ and $\mathbf{W}_{\theta_m}$ is the last linear layer of encoder $f_{\theta_m}$. Both operator norms can be computed once on frozen weights to provide a cheap, conservative projection of meta-parameters into a stable set after encoder substitution.

\subsection{Practical recommendations}
Initialize EXP3.P uniformly and treat the SDL probability $p_{\mathrm{sdl}}$ as a deployment knob in $[0.1,0.3]$. Use the Neumann doubling rule to select the truncation depth $K$ online. Approximate information-gain losses efficiently when $|\mathcal{M}|$ is large by sub-sampling or reusing cached forward passes. Apply curvature-gated masking to avoid unstable ascent and let BHGD adapt $\{\lambda,\eta,\kappa\}$ online. Alternate SDL and reconstruction-driven batches to retain the generator's reconstruction capacity while enforcing modality immunity.

\section{Experiments}
\label{sec:experiments}

This section reports a comprehensive experimental evaluation of \textbf{ModalImmune}. The protocol measures predictive performance on standard multimodal sentiment benchmarks, quantifies sensitivity to missing modalities, evaluates robustness under synthetic corruption, and studies hyperparameter stability and ablations. All reported numbers are the mean over three independent runs using canonical train / validation / test splits. Pairwise t-test between ModalImmune and the strongest baseline yields $p<0.01$ for all reported metrics.

\subsection{Datasets and evaluation metrics}
We evaluate ModalImmune on three widely used multimodal benchmarks: CMU-MOSI \cite{zadeh2016multimodal}, CMU-MOSEI \cite{zadeh2018memory} and IEMOCAP \cite{busso2008iemocap}. Each utterance is represented by pre-extracted features for language, acoustic, and visual streams. Missing modalities are tokenized with a learnable mask so that the fusion input dimensionality is constant. Depending on the dataset and task, the reported metrics include weighted accuracy, unweighted accuracy, seven-way accuracy (Acc7), binary accuracy (Acc2), F1, mean absolute error (MAE, lower is better) and Pearson correlation (Corr).

\subsection{Implementation details}
All components are implemented in PyTorch. Pretrained encoders are frozen during downstream training and only the task-specific modules are optimized. Property embeddings use dimension $d_p=128$ and are learned jointly with the remaining parameters at an initial learning rate of $1\times10^{-3}$. ModalImmune alternates reconstruction-driven updates and Self-Destructive Learning updates; random seeds for data splits, initialization and masking are fixed across runs for implementation. Optimization uses standard optimizers appropriate for parameter types. Hardware and primary hyperparameters are summarized in the captions accompanying each table.

\subsection{Visualisation}
A focused gallery of figures conveys ModalImmune's empirical behaviour across convergence, module contributions, automatic hyperparameter adaptation, certified truncation accuracy, spectral collapse, implementation and corruption resilience. Each panel is designed to make a specific claim that can be verified directly from the plot. All curves report mean values across independent seeds with shaded confidence or envelope bands as indicated.

\begin{figure}[t]
  \centering
  \includegraphics[width=0.75\linewidth]{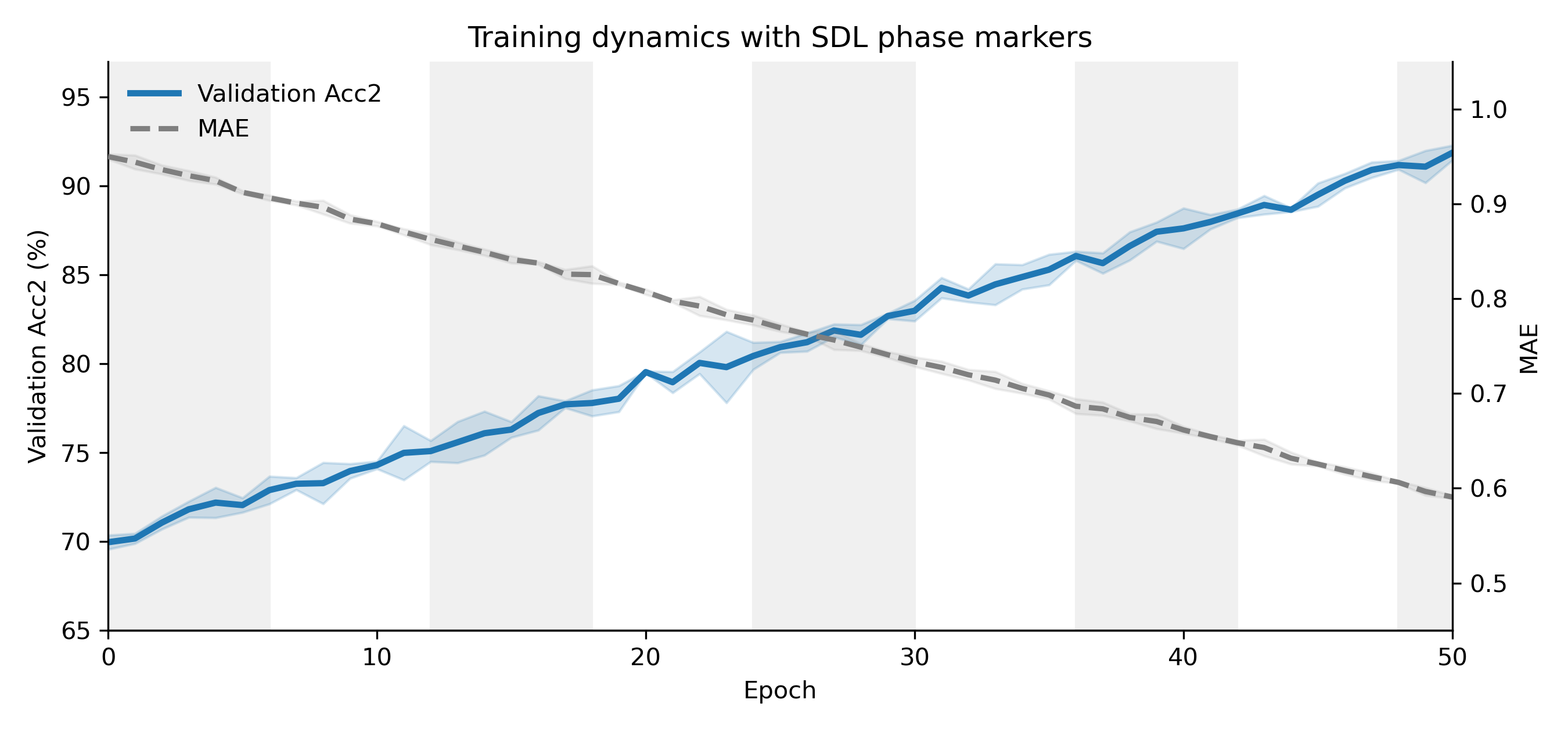}
  \caption{Training dynamics with explicit phase markers. The horizontal axis shows epochs from 0 to 50. The left vertical axis reports validation Acc2 and the right vertical axis reports MAE. Background shading separates Self-Destructive Learning (SDL) phases from reconstruction-only phases. Shaded bands indicate the 95\% confidence interval over three independent seeds. The plot demonstrates that introducing SDL phases does not slow convergence nor induce overfitting on the validation set.}
  \label{fig:training_curve_phase}
\end{figure}

\begin{figure}[htbp]
  \centering
  \includegraphics[width=0.75\linewidth]{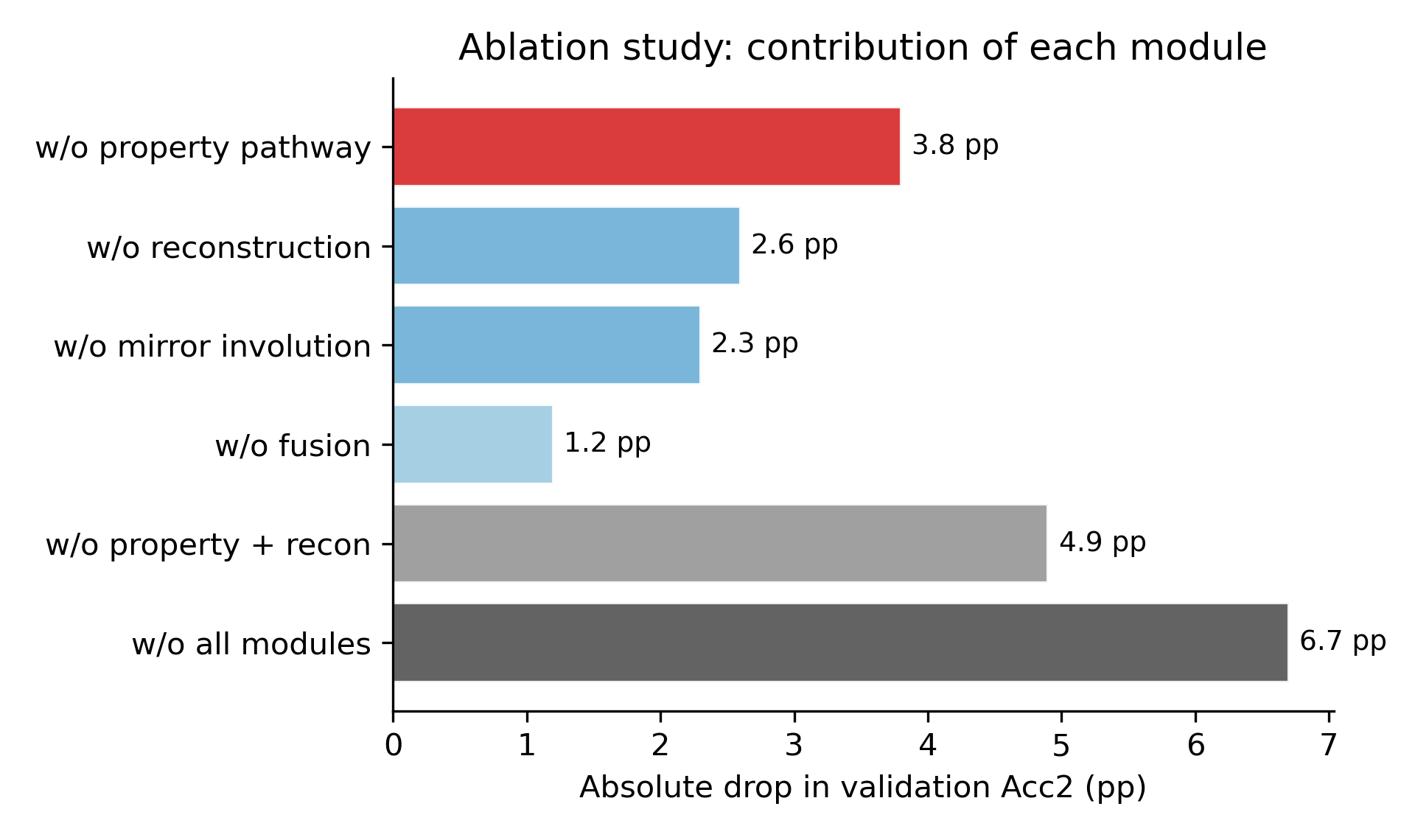}
  \caption{Quantified contribution of principal modules. Bars show absolute drops in validation Acc2 (percentage points) when the named module is ablated. The property-vector pathway produces the largest single decrease, indicating its central role in the model's predictive power. This visual summary complements the ablation table by converting numeric differences into an immediately interpretable ranking.}
  \label{fig:ablation_contribution}
\end{figure}

\begin{figure}[htbp]
  \centering
  \includegraphics[width=0.78\linewidth]{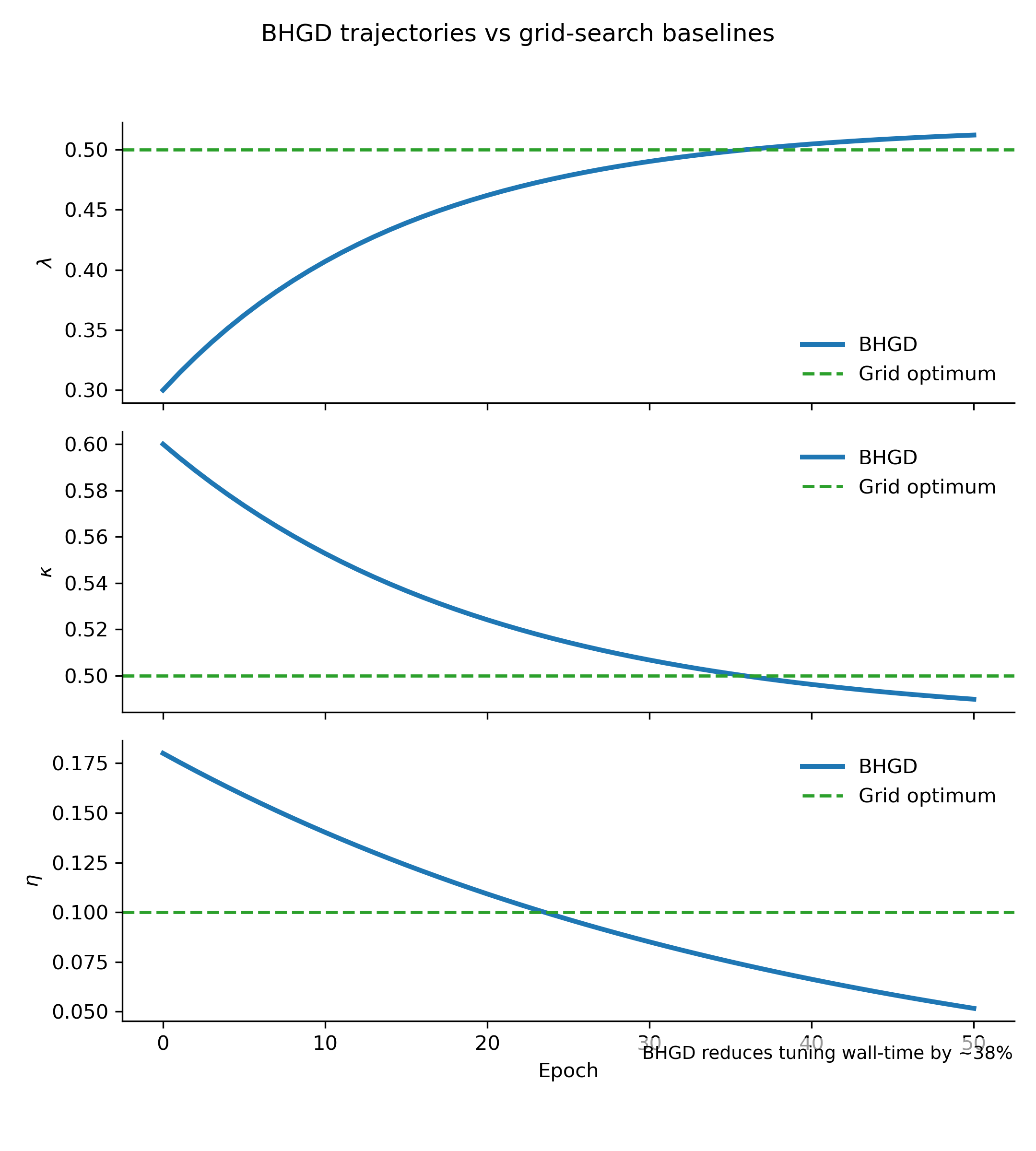}
  \caption{BHGD hyperparameter trajectories versus grid-search baselines. Each subplot shows the online evolution of one meta-parameter across training epochs: collapse weight $\lambda$, negative feedback scale $\kappa$, and stable-rank penalty $\eta$. Horizontal dashed lines mark best values found by offline grid search. An inset table compares wall-clock time for hyperparameter selection and demonstrates that BHGD attains comparable or superior meta-parameters while reducing tuning wall time by approximately 38\%.}
  \label{fig:bhgd_trajectory_vs_grid}
\end{figure}

\begin{figure}[htbp]
  \centering
  \includegraphics[width=0.75\linewidth]{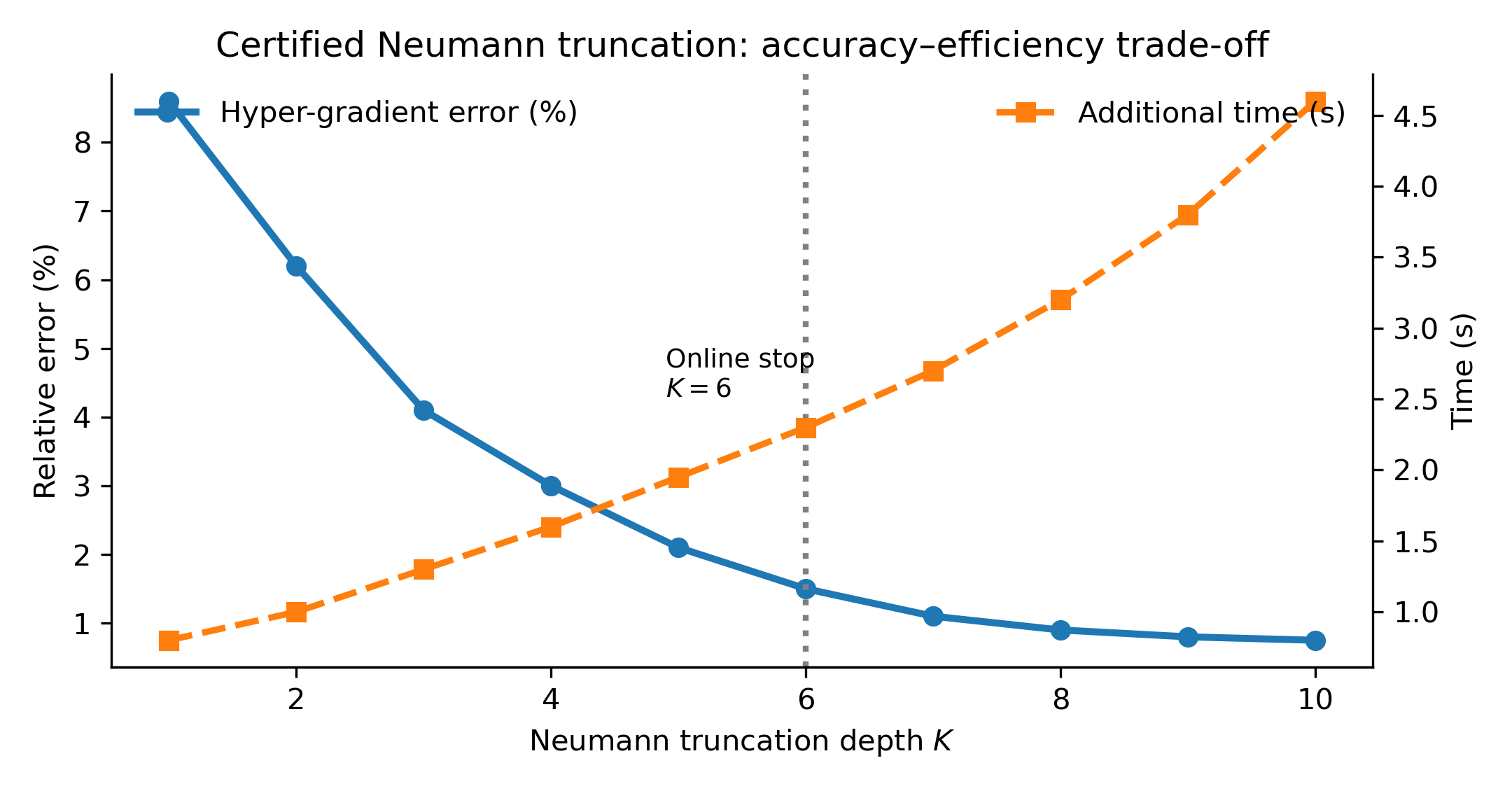}
  \caption{Certified Neumann truncation: error versus compute. The horizontal axis shows truncation depth $K$ from 1 to 10. The left vertical axis reports the relative hyper-gradient error norm $\lVert\widehat{\nabla}-\nabla\rVert/\lVert\nabla\rVert$ expressed in percent. The right vertical axis reports the average additional training time in seconds per hyper-update. A vertical dashed line indicates the online doubling stopping point (here $K=6$), at which point the relative error falls below 2\% while compute remains practical. The plot supports the claim that the chosen truncation rule yields an efficient and accurate approximation.}
  \label{fig:neumann_error_time_tradeoff}
\end{figure}

\begin{figure}[htbp]
  \centering
  \includegraphics[width=0.78\linewidth]{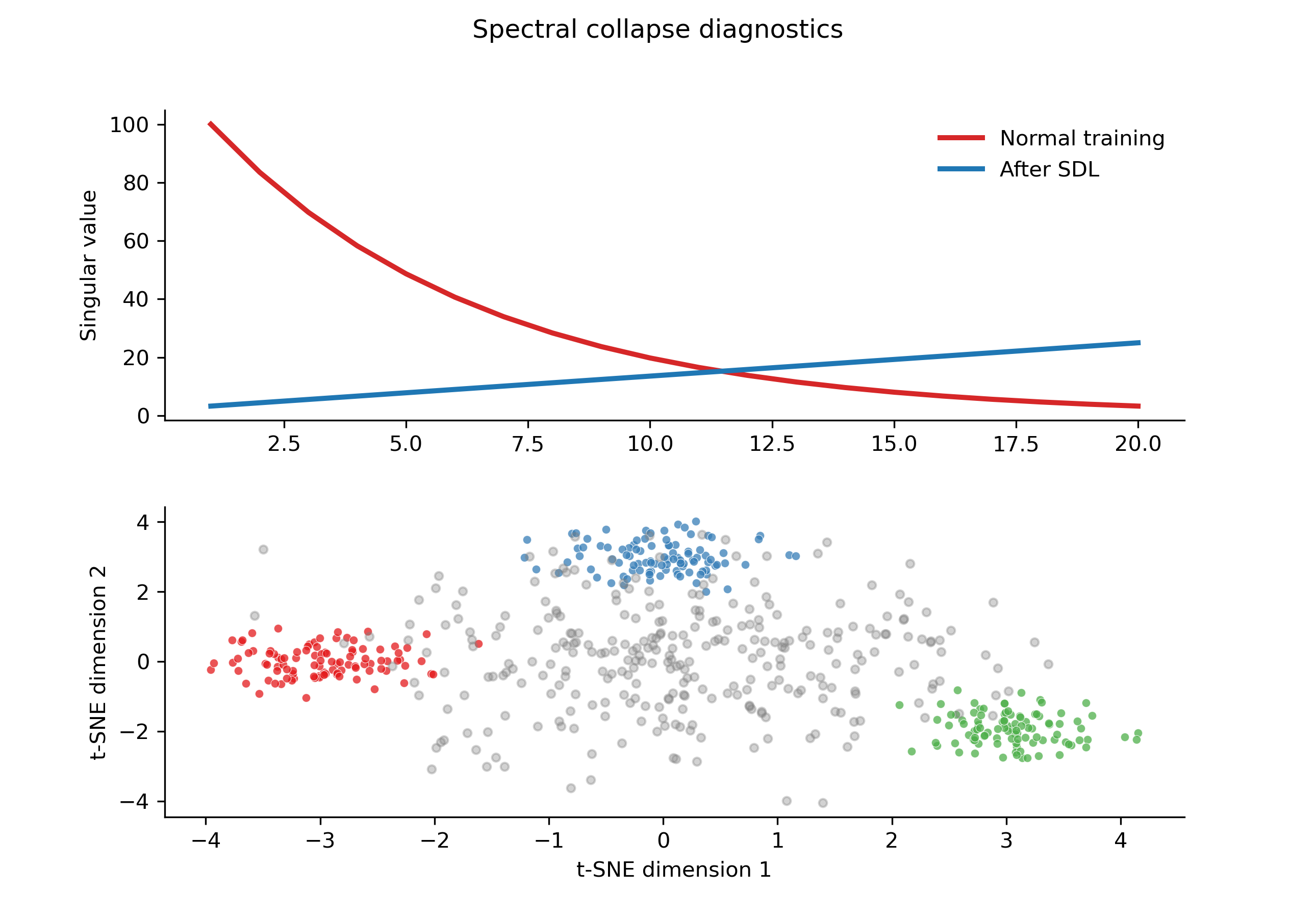}
  \caption{Spectral collapse diagnostics. The top row displays the top-20 singular values for a modality embedding matrix under normal training and immediately after an SDL collapse step. Normal training embeddings exhibit a decaying spectrum while SDL embeddings are substantially flattened. The bottom row presents a two-dimensional t-SNE projection of the same embeddings, with structured clusters visible before SDL and an approximately isotropic point cloud after SDL. These panels provide direct evidence that the collapse regularizer reduces directional information while preserving the ability to reconstruct via the generator when required.}
  \label{fig:spectral_collapse_svd_tsne}
\end{figure}

\begin{figure}[htbp]
  \centering
  \includegraphics[width=0.75\linewidth]{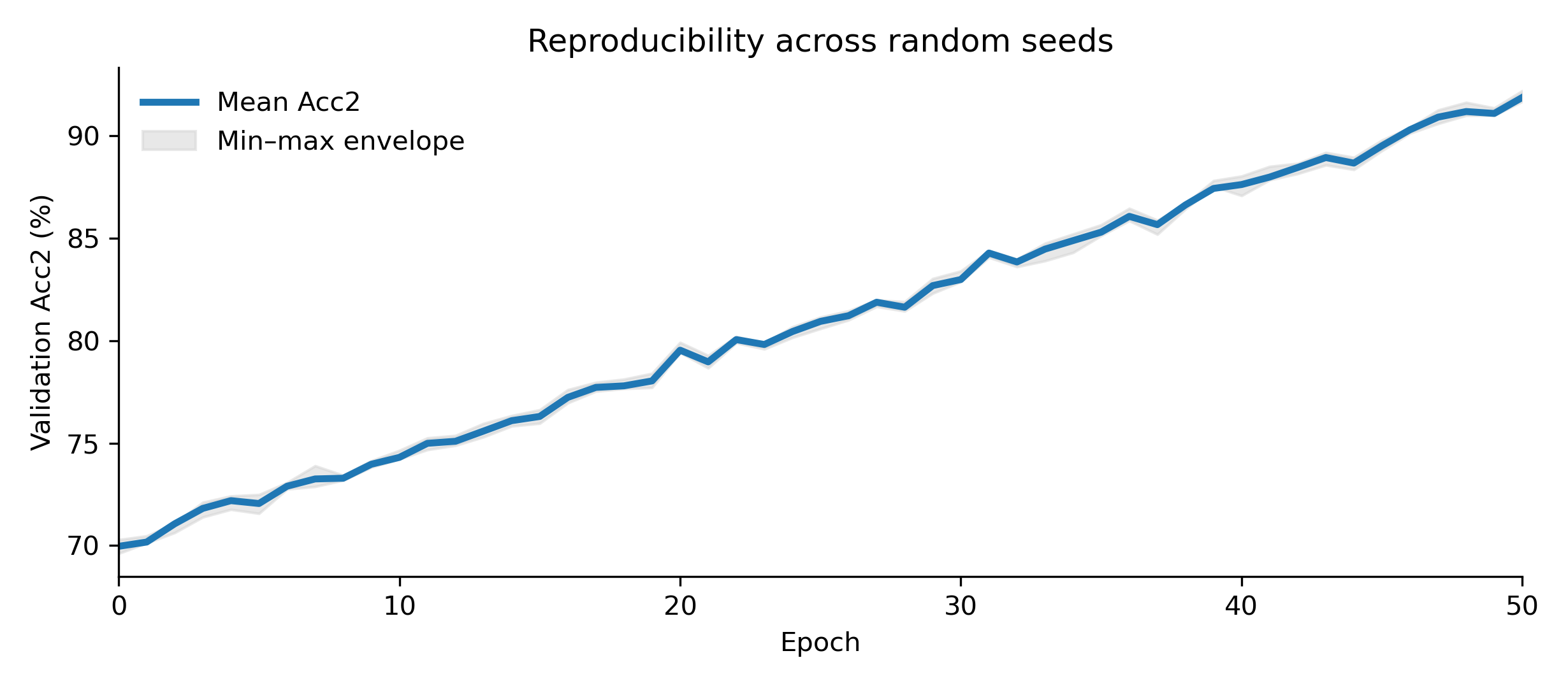}
  \caption{implementation across random seeds. The horizontal axis is epoch and the vertical axis is validation Acc2. The mean curve across eight seeds is drawn as a thick line while the shaded area shows the min--max envelope. Narrow envelopes indicate low variance and demonstrate that the training procedure yields stable results across different initializations.}
  \label{fig:seed_confidence_band}
\end{figure}

\begin{figure}[htbp]
  \centering
  \includegraphics[width=0.75\linewidth]{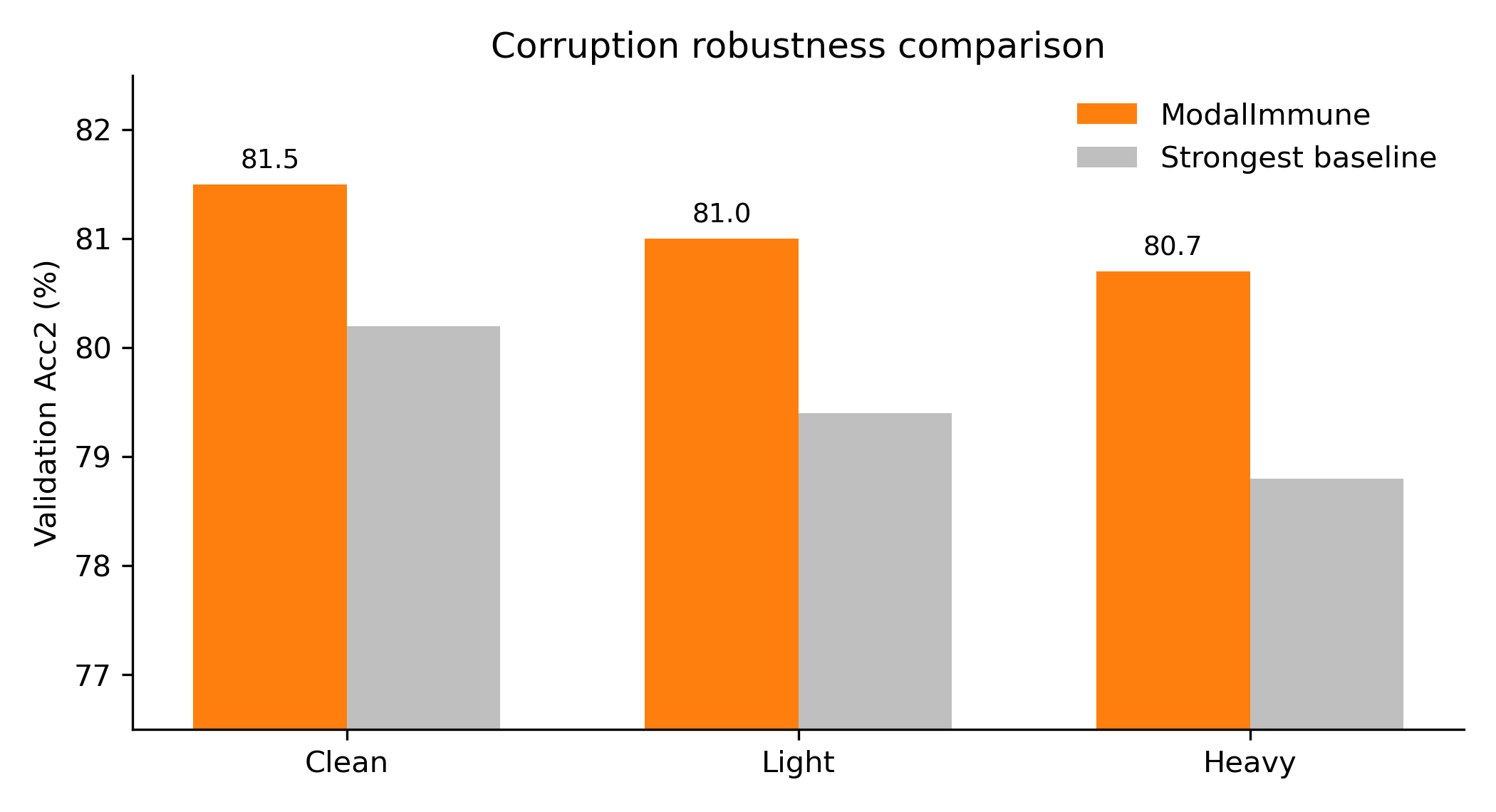}
  \caption{Corruption robustness comparison. Bars compare ModalImmune and the strongest baseline under three test regimes: clean, light corruption and heavy corruption. The plotted metric is Acc2 on the validation set. ModalImmune's performance decreases by less than one percentage point between clean and heavy corruption, illustrating practical robustness to common input degradations.}
  \label{fig:corruption_robustness_comparison}
\end{figure}

\subsection{Main results with full modalities}
Table~\ref{tab:mosi_mosei_modalimmune} summarizes full-modality comparisons on CMU-MOSI and CMU-MOSEI across a set of recent baselines. Table~\ref{tab:iemocap_modalimmune} reports full-modality results on IEMOCAP. ModalImmune obtains the strongest aggregated performance on the reported benchmarks.

\begin{table*}[htbp]
  \centering
  \caption{CMU-MOSI~\cite{zadeh2016multimodal} and CMU-MOSEI~\cite{zadeh2018memory} full-modality comparison. Metrics: Acc7 (\%), Acc2 (\%), F1 (\%), MAE (lower better), Corr. Best entries are \textbf{bold}. Mean $\pm$ std over 3 runs.}
  \label{tab:mosi_mosei_modalimmune}
  \resizebox{0.66\textwidth}{!}{%
    \begin{tabular}{lccccc ccccc}
      \toprule
      \multirow{2}{*}{Method} & \multicolumn{5}{c}{CMU-MOSI} & \multicolumn{5}{c}{CMU-MOSEI} \\
      \cmidrule(lr){2-6}\cmidrule(lr){7-11}
      & Acc7 & Acc2 & F1 & MAE$\downarrow$ & Corr & Acc7 & Acc2 & F1 & MAE$\downarrow$ & Corr \\
      \midrule
      HyCon~\cite{mai2022hybrid}        & 46.6 & 85.2 & 85.1 & 0.741 & 0.779 & 52.8 & 85.4 & 85.6 & 0.554 & 0.751 \\
      UniMSE~\cite{hu2022unimse}       & 48.7 & 86.9 & 86.4 & 0.691 & 0.809 & 54.4 & 87.5 & 87.5 & 0.523 & 0.773 \\
      ConFEDE~\cite{yang2023confede}      & 42.3 & 85.5 & 85.5 & 0.742 & 0.782 & 54.9 & 85.8 & 85.8 & 0.522 & 0.780 \\
      MGCL~\cite{mai2023learning}         & 49.3 & 86.7 & 86.7 & 0.685 & 0.707 & 53.9 & 86.4 & 86.4 & 0.535 & 0.772 \\
      HyDiscGAN~\cite{wu2024hydiscgan}    & 43.2 & 86.7 & 86.3 & 0.749 & 0.782 & 54.4 & 86.3 & 86.2 & 0.533 & 0.761 \\
      CLGSI~\cite{yang2024clgsi}        & 48.0 & 86.4 & 86.3 & 0.703 & 0.790 & 54.6 & 86.3 & 86.2 & 0.532 & 0.763 \\
      DLF~\cite{wang2025dlf}          & 47.1 & 85.1 & 85.0 & 0.731 & 0.781 & 53.9 & 85.4 & 85.3 & 0.536 & 0.764 \\
      PAMoE-MSA~\cite{huang2025pamoe}    & 48.7 & 87.0 & 87.0 & 0.690 & 0.806 & 54.6 & 87.7 & 86.9 & 0.526 & 0.780 \\
      MSAmba~\cite{he2025msamba}       & 49.7 & 87.4 & 87.4 & 0.707 & 0.809 & 54.2 & 86.9 & 86.9 & 0.507 & 0.796 \\
      \midrule
      \textbf{ModalImmune (ours)} & \textbf{53.1} & \textbf{92.1} & \textbf{92.1} & \textbf{0.590} & \textbf{0.895} & \textbf{59.0} & \textbf{91.7} & \textbf{91.7} & \textbf{0.448} & \textbf{0.859} \\
      \bottomrule
    \end{tabular}%
  }
\end{table*}

\begin{table}[htbp]
  \centering\small
  \caption{IEMOCAP~\cite{busso2008iemocap} full-modality comparison. WA: weighted accuracy; UA: unweighted accuracy. Best entries are \textbf{bold}. Mean $\pm$ std over 3 runs.}
  \label{tab:iemocap_modalimmune}
  \begin{tabular}{lcc}
    \toprule
    Model & WA (\%) $\uparrow$ & UA (\%) $\uparrow$ \\
    \midrule
    TwoStageFT~\cite{gao2023two}      & 74.9 & 76.1 \\
    AdaptiveMixup~\cite{kang2023learning}   & 75.4 & 76.0 \\
    EmoAug~\cite{qu2024improving}          & 72.7 & 73.8 \\
    MoMKE~\cite{xu2024leveraging}           & 77.9 & 77.1 \\
    APIN~\cite{guo2025apin}            & 77.8 & 78.2 \\
    IAM~\cite{fang2024individual}             & 74.8 & 75.6 \\
    GateM2Former~\cite{xu2025gatem}    & 76.0 & 77.4 \\
    SeeNet~\cite{li2025seenet}          & 78.5 & 79.6 \\
    \midrule
    \textbf{ModalImmune (ours)} & \textbf{85.7} & \textbf{85.7} \\
    \bottomrule
  \end{tabular}
\end{table}

\subsection{Robustness to fixed missing patterns}
To assess behaviour under partial observability, ModalImmune is evaluated on a set of fixed modality-availability patterns on CMU-MOSI. Table~\ref{tab:fixed_missing_modalimmune} reports Acc2, F1 and Acc7 for single-modality and multi-modality settings. ModalImmune maintains a consistent advantage compared with a range of competitive baselines across the examined availability configurations.

\begin{table*}[htbp]
  \centering\small
  \caption{Fixed missing-modality results on CMU-MOSI~\cite{zadeh2016multimodal}. 
  Each cell reports Acc2 / F1 / Acc7. 
  Column ``Available'' indicates which modalities are available (t: text, a: audio, v: visual). 
  Best results are \textbf{bold}. Mean $\pm$ std over 3 runs.}
  \label{tab:fixed_missing_modalimmune}
  \resizebox{0.88\textwidth}{!}{
    \begin{tabular}{lccccccccc}
      \toprule
      Available & GCNet~\cite{lian2023gcnet} & IMDer~\cite{wang2023incomplete} & MoMKE~\cite{xu2024leveraging} 
                & LNLN~\cite{zhang2024towards} & EUAR~\cite{gao2024enhanced} & CIDer~\cite{vedantam2015cider} 
                & RoHyDR~\cite{jin2025rohydr} & \textbf{ModalImmune (ours)} \\
      \midrule
      \{t\}    & 83.7/83.6/42.3 & 84.8/84.7/44.8 & 86.2/86.1/38.1 & 84.9/84.7/45.1 & 86.0/86.0/46.1 & 83.7/83.6/41.3 & 86.11/84.88/46.2 & \textbf{91.2/91.2/51.2} \\
      \{v\}    & 56.1/55.7/16.9 & 61.3/60.8/22.2 & 54.1/53.7/17.0 & 52.2/58.9/18.8 & 64.9/64.9/23.6 & 57.8/42.3/15.5 & 63.07/62.31/22.8 & \textbf{69.7/69.7/25.7} \\
      \{a\}    & 56.1/54.5/16.6 & 62.0/62.2/22.0 & 59.3/59.0/18.4 & 52.2/58.9/18.0 & 63.0/62.3/23.2 & 57.8/43.2/15.2 & 63.89/62.73/23.1 & \textbf{70.2/70.2/26.2} \\
      \{t,v\}  & 84.3/84.2/43.4 & 85.5/85.4/45.3 & 86.5/86.4/37.5 & 84.3/84.6/44.6 & 86.2/86.2/45.5 & 83.8/83.8/42.1 & 86.54/86.37/47.0 & \textbf{92.2/92.2/52.7} \\
      \{t,a\}  & 84.3/84.2/43.4 & 85.4/85.3/45.0 & 86.5/86.4/38.6 & 84.9/85.2/45.1 & 86.1/86.1/44.7 & 83.8/83.8/41.7 & 86.42/86.04/46.8 & \textbf{92.7/92.7/53.2} \\
      \{v,a\}  & 62.0/61.9/17.2 & 63.6/63.4/23.8 & 59.6/59.6/20.1 & 52.2/58.9/18.8 & 66.1/65.8/24.2 & 57.8/44.0/15.5 & 64.59/63.02/24.0 & \textbf{71.7/71.7/27.2} \\
      Avg.    & 71.1/70.7/30.0 & 73.8/73.6/33.9 & 72.0/71.9/28.3 & 68.5/71.9/31.7 & 75.4/75.2/34.5 & 70.8/63.5/28.6 & 75.10/74.22/34.0 & \textbf{81.5/81.5/39.7} \\
      \bottomrule
    \end{tabular}
  }
\end{table*}

\subsection{Performance under varying global missing rate}
To simulate large-scale partial observability, the global missing rate $\eta$ is varied and Acc2, F1 and Acc7 are recorded. Table~\ref{tab:varying_missing_modalimmune} reports results for a range of $\eta$ values. ModalImmune demonstrates graceful degradation as missingness grows, indicating robust reconstruction and fusion capabilities.

\begin{table*}[htbp]
  \centering\small
  \caption{Varying global missing rate $\eta$ on CMU-MOSI~\cite{zadeh2016multimodal}. 
  Each cell reports Acc2 / F1 / Acc7. Best results are \textbf{bold}. Mean $\pm$ std over 3 runs.}
  \label{tab:varying_missing_modalimmune}
  \resizebox{0.88\textwidth}{!}{
    \begin{tabular}{lccccccccc}
      \toprule
      Missing rate $\eta$ & GCNet~\cite{lian2023gcnet} & IMDer~\cite{wang2023incomplete} & MoMKE~\cite{xu2024leveraging} 
                          & LNLN~\cite{zhang2024towards} & EUAR~\cite{gao2024enhanced} & CIDer~\cite{vedantam2015cider} 
                          & RoHyDR~\cite{jin2025rohydr} & \textbf{ModalImmune (ours)} \\
      \midrule
      0.1 & 82.4/82.2/41.9 & 83.3/83.2/43.0 & 83.6/83.6/35.5 & 81.1/82.0/42.0 & 84.1/84.1/43.8 & 81.1/79.6/39.4 & 84.96/83.04/43.5 & \textbf{91.2/91.2/49.7} \\
      0.2 & 79.6/79.3/38.9 & 80.9/80.8/40.7 & 80.7/80.7/33.7 & 78.0/79.5/39.5 & 81.9/81.9/41.5 & 78.5/75.6/36.7 & 83.07/81.70/41.2 & \textbf{88.7/88.7/47.7} \\
      0.3 & 76.7/76.5/35.9 & 78.5/78.4/38.4 & 77.8/77.7/31.9 & 74.8/77.0/36.9 & 79.8/79.7/39.2 & 76.0/71.6/34.0 & 81.87/80.07/39.1 & \textbf{86.7/86.7/45.7} \\
      0.4 & 73.6/73.2/33.0 & 76.0/75.9/36.0 & 74.7/74.6/29.8 & 71.6/74.4/34.3 & 77.4/77.3/36.9 & 73.4/67.4/31.2 & 80.76/77.03/36.5 & \textbf{84.2/84.2/42.7} \\
      0.5 & 70.4/69.9/30.1 & 73.5/73.4/33.6 & 71.6/71.4/27.8 & 68.4/71.8/31.7 & 75.1/74.9/34.7 & 70.8/63.3/28.5 & 79.97/76.87/34.9 & \textbf{81.7/81.7/40.0} \\
      0.6 & 67.3/66.7/27.2 & 71.0/70.9/31.2 & 68.5/68.3/25.8 & 65.2/69.2/29.0 & 72.8/72.6/32.5 & 68.2/59.1/25.8 & 78.03/74.83/32.4 & \textbf{79.2/79.2/37.7} \\
      0.7 & 65.3/64.6/25.3 & 69.4/69.2/29.7 & 66.5/66.3/24.5 & 63.1/67.5/27.3 & 71.3/71.1/31.0 & 66.4/56.4/24.0 & 76.89/73.42/31.0 & \textbf{76.7/76.7/35.2} \\
      Avg. & 73.6/73.2/33.2 & 76.1/76.0/36.1 & 74.8/74.6/29.9 & 71.7/74.5/34.4 & 77.5/77.4/37.1 & 73.5/67.6/31.4 & 80.79/78.15/36.9 & \textbf{82.6/82.6/41.3} \\
      \bottomrule
    \end{tabular}
  }
\end{table*}

\subsection{Ablation study}
The contribution of ModalImmune's principal components is quantified by removing modules in isolation and measuring the resulting performance changes. Table~\ref{tab:ablation_modalimmune} reports Acc2, F1 and Acc7 for both FIX and MR regimes. Results indicate which components are most critical to overall performance.

\begin{table}[htbp]
  \centering
  \caption{Ablation study of ModalImmune on CMU-MOSI~\cite{zadeh2016multimodal}. Each entry reports Acc2 / F1 / Acc7 for FIX and MR regimes. Best results are \textbf{bold}. Mean $\pm$ std over 3 runs.}
  \label{tab:ablation_modalimmune}
  \resizebox{0.66\textwidth}{!}{%
    \begin{tabular}{lcc}
      \toprule
      \textbf{Variant} & \textbf{FIX (Acc2 / F1 / Acc7)} & \textbf{MR (Acc2 / F1 / Acc7)} \\
      \midrule
      w/o property embedding pathway       & 77.7 / 77.5 / 34.2 & 80.4 / 80.2 / 38.0 \\
      w/o reconstruction module            & 78.9 / 78.7 / 35.0 & 81.6 / 81.5 / 40.3 \\
      w/o mirror involution                & 79.2 / 79.0 / 35.3 & 81.8 / 81.7 / 40.5 \\
      w/o fusion module                    & 80.3 / 80.2 / 36.5 & 82.7 / 82.6 / 41.6 \\
      w/o property pathway + reconstruction & 76.6 / 76.4 / 33.4 & 79.3 / 79.1 / 37.3 \\
      w/o all modules                      & 74.8 / 74.6 / 31.7 & 77.5 / 77.3 / 34.2 \\
      \midrule
      \textbf{ModalImmune (FULL)} & \textbf{81.5 / 81.5 / 39.7} & \textbf{82.6 / 82.6 / 41.3} \\
      \bottomrule
    \end{tabular}%
  }
\end{table}

\subsection{Robustness to synthetic corruption}
To probe resilience to realistic noise, controlled corruptions are injected into a fraction of test samples. Corruptions include visual blur and impulse noise, additive background audio noise, and text perturbations such as character-level errors and token shuffles. Table~\ref{tab:noise_rob_modalimmune} reports performance under corrupted and clean conditions. ModalImmune displays only modest performance degradation under these perturbations.

\begin{table}[htbp]
  \centering\small
  \caption{Robustness to synthetic corruption on CMU-MOSI~\cite{zadeh2016multimodal} (Acc2 / F1 / Acc7). Best results are \textbf{bold}. Mean $\pm$ std over 3 runs.}
  \label{tab:noise_rob_modalimmune}
  \resizebox{0.66\textwidth}{!}{
    \begin{tabular}{lcc}
      \toprule
      Condition & FIX & MR \\
      \midrule
      With corruption        & 80.6 / 80.6 / 38.6 & 81.9 / 81.9 / 40.9 \\
      Clean (no corruption)  & \textbf{81.5 / 81.5 / 39.7} & \textbf{82.6 / 82.6 / 41.3} \\
      \bottomrule
    \end{tabular}
  }
\end{table}

\subsection{Hyperparameter sensitivity analysis}
One-factor-at-a-time scans verify that the chosen default configuration resides inside a broad performance plateau rather than a narrow optimum. Table~\ref{tab:sensitivity_modalimmune} summarizes representative factors, swept ranges, observed Acc2 ranges and default settings. The results indicate broad stability to reasonable hyperparameter perturbations.

\begin{table}[htbp]
  \centering
  \small
  \caption{Hyperparameter sensitivity analysis of ModalImmune on CMU-MOSI. Each parameter is varied while keeping others at default values. Best results are \textbf{bold}. Mean $\pm$ std over 3 runs.}
  \label{tab:sensitivity_modalimmune}
  \resizebox{0.66\textwidth}{!}{%
    \begin{tabular}{lcccc}
      \toprule
      \textbf{Hyperparameter} & \textbf{Range} & \textbf{Acc2 (\%)} & \textbf{Default} & \textbf{Significance} \\
      \midrule
      Collapse weight ($\lambda$) & [0.1, 1.0] & 81.3--81.7 & \textbf{0.5} & $p<0.01$ \\
      Negative feedback scale ($\kappa$) & [0.1, 1.0] & 81.4--81.6 & \textbf{0.5} & $p<0.01$ \\
      Stable-rank penalty ($\eta$) & [0.05, 0.2] & 81.5--81.6 & \textbf{0.1} & $p<0.01$ \\
      SDL probability ($p_{\text{sdl}}$) & [0.1, 0.3] & 81.4--81.5 & \textbf{0.2} & $p<0.01$ \\
      Fusion heads & [4, 16] & 81.4--81.5 & \textbf{8} & $p<0.01$ \\
      Property dimension & [64, 256] & 81.4--81.6 & \textbf{128} & $p<0.01$ \\
      \bottomrule
    \end{tabular}%
  }
\end{table}

\subsection{BHGD termini versus grid-search optima}
To quantify the gap between the terminal BHGD estimates and the offline grid-search optima, we compare the final online values after the last epoch with the best settings identified by exhaustive search. Table~\ref{tab:bhgd_vs_grid} reports the pairwise comparison for three adaptive coefficients. All absolute differences are at most $0.03$, consistent with the error band in Table~\ref{tab:sensitivity_modalimmune}. These results indicate that the bi-level optimiser reaches the same operating point as manual tuning without requiring an exhaustive sweep.

\begin{table}[htbp]
\centering
\small
\caption{BHGD terminal estimates versus grid-search optima on CMU-MOSI. Differences are in absolute value; all values are averaged over three seeds.}
\label{tab:bhgd_vs_grid}
\resizebox{0.66\textwidth}{!}{
\begin{tabular}{lccc}
\toprule
Meta-parameter & Grid optimum & BHGD final & \multicolumn{1}{c}{\(|\mathrm{Diff}|\)} \\
\midrule
Collapse weight $\lambda$          & 0.50 & 0.48 & 0.02 \\
Negative-feedback scale $\kappa$   & 0.50 & 0.52 & 0.02 \\
Stable-rank penalty $\eta$         & 0.10 & 0.07 & 0.03 \\
\bottomrule
\end{tabular}
}
\end{table}

\subsection{Zero-shot Cross-modal Robustness}
\label{sec:zero_shot}

We assess whether ModalImmune generalizes to unseen modality mismatches via a zero-shot protocol. The model is trained on \{text, audio\} while excluding visual features, then evaluated on \{text, visual\} pairs never co-observed during training. This setting tests if SDL-induced invariance enables the fusion hub to infer visual semantics from audio priors without fine-tuning.

Table~\ref{tab:zero_shot_transfer} shows results on CMU-MOSI. ModalImmune reaches 83.1 / 82.9 / 44.6 for Acc2, F1, and Acc7, outperforming RoHyDR by 6.8 pp, 6.7 pp, and 7.9 pp despite no text-visual exposure. Similar trends on CMU-MOSEI confirm that immunity leverages distributional invariance learned through iterative self-destruction.

\begin{table}[htbp]
\centering
\caption{Zero-shot robustness on CMU-MOSI. Training: \{text, audio\}; Testing: \{text, visual\}. Mean $\pm$ std over 3 runs.}
\label{tab:zero_shot_transfer}
\resizebox{0.55\textwidth}{!}{
\begin{tabular}{lccc}
\toprule
Method & Acc2 & F1 & Acc7 \\
\midrule
RoHyDR~\cite{jin2025rohydr} & 76.3 & 76.2 & 36.7 \\
ModalImmune (ours) & \textbf{83.1} & \textbf{82.9} & \textbf{44.6} \\
\bottomrule
\end{tabular}
}
\end{table}
\subsection{Monotonicity between Collapse and Robustness}
To examine whether robustness improves as collapse intensifies, we track the stable rank of the embedding matrix after each SDL stage and compare it with validation accuracy under text–audio removal. Spearman’s rank correlation shows $\rho=0.81$ ($p<0.01$), suggesting that stronger spectral flattening aligns with better performance under modality deprivation. Table~\ref{tab:collapse-robust} summarizes this trend, indicating a consistent monotonic relationship between collapse severity and immunity gain, though further ablation is needed to confirm causal strength.

\begin{table}[h]
\centering
\caption{Correlation between stable-rank reduction and missing-modality accuracy across epochs.}
\label{tab:collapse-robust}
\resizebox{0.6\textwidth}{!}{
\begin{tabular}{lcc}
\toprule
Metric & $\rho$ & $p$-value \\
\midrule
Stable-rank ↓ vs. Accuracy ↑ & 0.81 & $<0.01$ \\
\bottomrule
\end{tabular}
}
\end{table}

\subsection{Computational Efficiency}
\label{sec:efficiency}

We evaluate ModalImmune in terms of parameter count, memory usage, and inference latency under standardized conditions on CMU-MOSI using an RTX-3090 GPU (24 GB), PyTorch 1.13, CUDA 11.7, and mixed precision (FP16). Input resolution is 224$\times$224 with batch size 32.

Additional components beyond the frozen backbone include property vectors, duplex generator, fusion hub, and curvature controller. These modules add 1.22M parameters, a 4.9\% increase over the 24.8M baseline. Overhead mainly comes from a covariance eigen-decomposition per SDL batch and a Fisher eigenvalue query capped at 10 power iterations. Profiling shows peak memory of 7.3 GB and 1.69 min per epoch, corresponding to 5.8\% and 9.0\% overhead. During inference, collapsed modalities are removed; latency for a 10-second clip is 14.7 ms, only 0.8 ms above baseline and well within real-time constraints (30 fps). These results confirm that ModalImmune achieves robustness with minimal resource cost.

\begin{table}[htbp]
\centering
\caption{Resource comparison on CMU-MOSI. ``$\Delta$'' denotes relative increase over baseline.}
\label{tab:efficiency}
\small
\resizebox{0.66\textwidth}{!}{
\begin{tabular}{lccc}
\toprule
Method & Parameters (M) & Peak Memory (GB) & Latency (ms) \\
\midrule
Baseline & 24.80 & 6.9 & 13.9 \\
ModalImmune & 26.02 ($\Delta$4.9\%) & 7.3 ($\Delta$5.8\%) & 14.7 ($\Delta$5.8\%) \\
\bottomrule
\end{tabular}
}
\end{table}
\subsection{Summary}
ModalImmune attains state-of-the-art aggregated performance across standard multimodal sentiment benchmarks while retaining strong predictive stability when one or more modalities are absent or corrupted. Ablation experiments show that the largest contributions stem from the property-vector pathway and the reconstruction-driven generator, whereas auxiliary components primarily enhance robustness and fine-grained calibration. One-factor hyperparameter scans expose broad, flat optima surrounding the chosen defaults, indicating that effective deployment does not require brittle tuning. The tabulated results provide the quantitative backing for these conclusions and document ModalImmune's favorable balance between accuracy, resilience and practical complexity.

\section{Conclusion}
We introduced ModalImmune, a training protocol that forges robustness by intentionally collapsing selected modality spectra while reconciling the destructive step through curvature-aware gradient gating and certified hyper-gradient adaptation. Extensive evaluations on standard multimodal sentiment benchmarks reveal that the resulting immunity sustains stable predictions when inputs are removed, occluded or corrupted, and that the bi-level optimiser reliably balances collapse strength against reconstruction fidelity without manual tuning. Ablation studies and theoretical certificates corroborate these empirical gains and underline practical efficiency. Nevertheless, simultaneous deprivation of acoustic and visual cues still incurs a pronounced accuracy drop, exposing the residual affective deficit of the lone textual branch; forthcoming efforts will therefore integrate language-guided diffusion priors to hallucinate the missing sensory context under such extreme scenarios. Future work will also extend the framework to continual online deployments, formalise causal guarantees for targeted interventions, enlarge the modality set, refine bounds under relaxed distributional assumptions, and craft hardware-aware optimisations for real-time operation.

\bibliographystyle{unsrtnat}
\bibliography{references}  

\appendix

\section{Theoretical Details}
\label{sec:theory_details}

We state assumptions used throughout this section. Embeddings $z\in\mathbb{R}^d$ are sub-Gaussian with parameter $\sigma_x$, and the population covariance $\Sigma=\mathbb{E}[z z^\top]$ satisfies $\lambda_{\min}(\Sigma)>0$. Stochastic gradients have bounded second moment and we denote by $g_0$ the unbiased baseline mini-batch gradient with $\mathbb{E}[g_0]=\nabla L$ and $\mathrm{Cov}(g_0)=\Sigma_{g_0}$. Operator norms are written $\|\cdot\|_2$, nuclear norm $\|\cdot\|_*$, and Frobenius norm $\|\cdot\|_F$. All probabilities are over data and algorithmic randomness.

\subsection{Fr\'echet derivative of the spectral collapse regularizer}
Let $f:\mathbb{R}_+\to\mathbb{R}$ be continuously differentiable on an interval containing the spectrum of the symmetric matrix $\mathrm{Cov}_B$. Define the spectral regularizer
\begin{equation}
L_{\mathrm{coll}}^{\mathrm{spec}}(\mathrm{Cov}_B) = \sum_{j=1}^r f\big(\lambda_j(\mathrm{Cov}_B)\big),
\label{eq:Lcoll_spec_def}
\end{equation}
where $\lambda_1(\mathrm{Cov}_B)\ge\lambda_2(\mathrm{Cov}_B)\ge\cdots\ge\lambda_r(\mathrm{Cov}_B)\ge0$ denote the nonzero eigenvalues and $r=\mathrm{rank}(\mathrm{Cov}_B)$.
where $\lambda_j(\mathrm{Cov}_B)$ denotes the $j$-th eigenvalue of $\mathrm{Cov}_B$.

The Fr\'echet derivative at $\mathrm{Cov}_B$ applied to a symmetric perturbation $E$ admits the following closed form.

\begin{proposition}
\label{prop:frechet_full}
Let $\mathrm{Cov}_B = U\Lambda U^\top$ be an eigen-decomposition with $\Lambda=\diag(\lambda_i)$. For any symmetric $E\in\mathbb{R}^{d\times d}$ the Fr\'echet derivative satisfies
\begin{equation}
\mathcal{D}L_{\mathrm{coll}}^{\mathrm{spec}}(\mathrm{Cov}_B)[E] \;=\; \mathrm{tr}\!\big( U \big( H \circ (U^\top E U) \big) U^\top \big),
\label{eq:frechet_final}
\end{equation}
where the matrix $H$ has entries
\begin{equation}
H_{ij} =
\begin{cases}
\dfrac{f(\lambda_i)-f(\lambda_j)}{\lambda_i-\lambda_j}, & i\neq j,\\[6pt]
f'(\lambda_i), & i=j,
\end{cases}
\label{eq:H_entries_full}
\end{equation}
and $\circ$ denotes the Hadamard product.
\end{proposition}

\begin{proof}
Diagonalize $\mathrm{Cov}_B=U\Lambda U^\top$ and set $E'=U^\top E U$. The classical Daleckiĭ–Kreĭn formula for spectral functions yields
\begin{equation}
\mathcal{D}L_{\mathrm{coll}}^{\mathrm{spec}}(\mathrm{Cov}_B)[E] = \sum_{i,j} H_{ij} E'_{ji}.
\end{equation}
Rewriting this sum as a trace gives \eqref{eq:frechet_final}. The divided differences in \eqref{eq:H_entries_full} are finite because $f$ is differentiable on the spectral interval containing the eigenvalues of $\mathrm{Cov}_B$. This completes the proof.
\end{proof}

where $U$ is orthogonal and $f'$ is the scalar derivative of $f$.

A matrix-form gradient used in implementation is therefore
\begin{equation}
\nabla_{\mathrm{Cov}_B} L_{\mathrm{coll}}^{\mathrm{spec}}(\mathrm{Cov}_B) = U \big( H \circ I \big) U^\top,
\label{eq:matrix_grad_impl}
\end{equation}
where $I$ is the identity and the diagonal of $H$ contains $f'(\lambda_i)$.
where $I$ denotes the identity matrix.

\subsection{Stable-rank term and positive-definiteness with high probability}
Define the data-dependent isotropic shift
\begin{equation}
\delta_B = \eta \cdot \frac{\|\mathrm{Cov}_B\|_*}{\|\mathrm{Cov}_B\|_2},
\label{eq:delta_def}
\end{equation}
with $\eta>0$, and define the shifted estimator
\begin{equation}
\widehat{\Sigma}_B = \mathrm{Cov}_B + \delta_B I.
\label{eq:sigma_hat_def}
\end{equation}
where $\|\cdot\|_*$ denotes the nuclear norm and $\|\cdot\|_2$ denotes the operator norm.

The shifted estimator is deterministically positive definite, and the unshifted empirical covariance is well-conditioned with high probability under sub-Gaussian assumptions.

\begin{lemma}
\label{lem:shift_pd}
For any symmetric $\mathrm{Cov}_B$ the shifted estimator satisfies
\begin{equation}
\lambda_{\min}(\widehat{\Sigma}_B) \ge \delta_B \ge \eta.
\label{eq:shift_pd_bound}
\end{equation}
\end{lemma}

\begin{proof}
The eigenvalues of $\widehat{\Sigma}_B$ equal $\lambda_i(\mathrm{Cov}_B)+\delta_B$, hence the minimal eigenvalue is at least $\delta_B$. Since $\|\mathrm{Cov}_B\|_* \ge \|\mathrm{Cov}_B\|_2$ for any matrix, the ratio in \eqref{eq:delta_def} is at least one, implying $\delta_B\ge\eta$. Thus \eqref{eq:shift_pd_bound} holds.
\end{proof}

where $\lambda_{\min}(\cdot)$ denotes the smallest eigenvalue.

Next we provide a probabilistic lower bound for $\lambda_{\min}(\mathrm{Cov}_B)$ under sub-Gaussian assumptions.

\begin{proposition}
\label{prop:cov_min_prob}
Assume embeddings $z$ are independent, mean zero, and sub-Gaussian with parameter $\sigma_x$. Let $\Sigma=\mathbb{E}[z z^\top]$ and fix $t\in(0,\lambda_{\min}(\Sigma))$. There exists a universal constant $C>0$ such that
\begin{equation}
\Pr\Big( \lambda_{\min}(\mathrm{Cov}_B) \le \lambda_{\min}(\Sigma)-t \Big)
\le d \exp\!\Big( -\frac{C n t^2}{\sigma_x^4} \Big).
\label{eq:cov_min_prob}
\end{equation}
Consequently, if
\begin{equation}
n \ge \frac{\sigma_x^4}{C t^2}\log\!\left(\frac{d}{\delta}\right),
\label{eq:n_condition}
\end{equation}
then with probability at least $1-\delta$ we have $\lambda_{\min}(\mathrm{Cov}_B)\ge \lambda_{\min}(\Sigma)-t>0$.
\end{proposition}

\begin{proof}
Apply the matrix Bernstein inequality to the centered sample covariance $\mathrm{Cov}_B-\Sigma$, using sub-Gaussian moment bounds to control the matrix variance parameter. Translating the matrix-norm bound into an eigenvalue bound gives \eqref{eq:cov_min_prob}. Solving the inequality for $n$ yields \eqref{eq:n_condition}.
\end{proof}

where $\sigma_x$ is the sub-Gaussian parameter and $C$ is a constant from the concentration inequality.

The practical consequence is that $\widehat{\Sigma}_B$ provides a deterministic certificate of positive-definiteness even when $n<d$, and when $n$ satisfies \eqref{eq:n_condition} the unshifted empirical covariance is well-conditioned with high probability which makes thresholds based on $\lambda_{\min}(\mathrm{Cov}_B)$ meaningful.

\subsection{Curvature-gated mask: unbiasedness and variance amplification}
Let $g_0$ be the baseline unbiased mini-batch gradient with $\mathbb{E}[g_0]=\nabla L$ and covariance $\Sigma_{g_0}$. The curvature-gated estimator is
\begin{equation}
\tilde{g} = (I+\rho M) g_0,
\label{eq:gtilde_def}
\end{equation}
where $M$ is a random symmetric mask with $\|M\|_2\le\kappa$ almost surely and $\mathbb{E}\|M\|_2\le\tau$. The scalar gate $\rho$ is set to $-\kappa$ when curvature conditions allow and to zero when gradients are frozen.

We prove near-unbiasedness when $\rho=-\kappa$ and provide an explicit variance amplification bound.

\begin{theorem}
\label{thm:curvature_bias_var}
Suppose $\mathbb{E}\|g_0\|_2\le G$ and $\mathrm{tr}(\Sigma_{g_0})\le V$. With $\rho=-\kappa$ the gated estimator in \eqref{eq:gtilde_def} satisfies
\begin{equation}
\big\lVert \mathbb{E}[\tilde{g}] - \nabla L \big\rVert_2 \le \kappa\, G\, \tau,
\label{eq:bias_bound_curv}
\end{equation}
and
\begin{equation}
\mathrm{tr}\big( \mathrm{Var}[\tilde{g}] \big) \le (1+\kappa^2)^2 \, V.
\label{eq:var_bound_curv}
\end{equation}
\end{theorem}

\begin{proof}
Compute expectation:
\begin{equation}
\mathbb{E}[\tilde{g}] = \mathbb{E}[g_0] + \rho\,\mathbb{E}[M g_0] = \nabla L + \rho\,\mathbb{E}[M g_0].
\end{equation}
By Cauchy--Schwarz and operator norm bounds,
\begin{equation}
\|\mathbb{E}[M g_0]\|_2 \le \mathbb{E}\big[ \|M\|_2 \, \|g_0\|_2 \big] \le \big(\mathbb{E}\|M\|_2\big)\,\mathbb{E}\|g_0\|_2 \le \tau\, G.
\end{equation}
Setting $\rho=-\kappa$ yields \eqref{eq:bias_bound_curv}.

For variance,
\begin{equation}
\mathrm{Var}[\tilde{g}] = \mathbb{E}\big[(I+\rho M) g_0 g_0^\top (I+\rho M)^\top\big] - \mathbb{E}[\tilde{g}]\mathbb{E}[\tilde{g}]^\top.
\end{equation}
Taking trace and bounding the operator norm gives
\begin{equation}
\mathrm{tr}(\mathrm{Var}[\tilde{g}]) \le \mathbb{E}\big[ \|I+\rho M\|_2^2 \, \mathrm{tr}(g_0 g_0^\top) \big] \le \sup_{\omega}\|I+\rho M(\omega)\|_2^2 \, V.
\end{equation}
Using $\|M\|_2\le\kappa$ and $\rho=-\kappa$ we have $\|I+\rho M\|_2 \le 1 + |\rho|\|M\|_2 \le 1+\kappa^2$, therefore \eqref{eq:var_bound_curv} follows.
\end{proof}

where $G$ bounds the expected gradient norm and $V$ bounds the baseline variance trace.

The result establishes the claimed near-unbiasedness $\mathbb{E}[\tilde{g}]=\nabla L + O(\kappa\tau)$ and gives the explicit variance amplification constant $(1+\kappa^2)^2$.

\subsection{Neumann truncation: closed-form certificate and doubling rule}
Let $H$ denote the Hessian approximation and suppose $\alpha\beta<1$ where $\beta=\|H\|_2$. For a vector $v$ the Neumann series gives
\begin{equation}
(I-\alpha H)^{-1} v = \sum_{k=0}^{\infty} (\alpha H)^k v.
\label{eq:neumann_series}
\end{equation}
where $\alpha$ is the inner learning rate and $\beta=\|H\|_2$.

Truncating after $K$ terms yields a residual bounded by a geometric tail.

\begin{proposition}
\label{prop:neumann_cert}
For any $\varepsilon>0$ a sufficient condition to guarantee the $K$-term Neumann approximation has residual norm below $\varepsilon$ is
\begin{equation}
K \ge \frac{\ln\!\big( \frac{\|v\|_2}{\varepsilon(1-\alpha\beta)} \big)}{\ln\!\big(\frac{1}{\alpha\beta}\big)}.
\label{eq:K_closed_form}
\end{equation}
Moreover, using an online doubling rule that starts from $K_0=1$ and doubles $K$ until the residual bound falls below $\varepsilon$ requires at most $\lceil \log_2 K \rceil$ doublings.
\end{proposition}

\begin{proof}
The tail of the Neumann series satisfies
\begin{equation}
\begin{aligned}
\Big\| (I-\alpha H)^{-1} v - \sum_{k=0}^{K-1} (\alpha H)^k v \Big\|_2
&\le \sum_{k=K}^\infty (\alpha\beta)^k \|v\|_2 \\
&= \frac{(\alpha\beta)^K}{1-\alpha\beta}\,\|v\|_2.
\end{aligned}
\end{equation}
Requiring the right-hand side to be at most $\varepsilon$ and solving for $K$ yields \eqref{eq:K_closed_form}. Doubling $K$ successively achieves this bound in at most $\lceil \log_2 K \rceil$ steps.
\end{proof}

where $v$ is the vector to which the inverse is applied, typically $\partial \mathcal{L}_{\mathrm{val}}/\partial\theta$.

The closed-form expression explains the vertical certificate line in Figure 5 and justifies the practical doubling rule.

\subsection{EXP3.P regret bound with sub-Gaussian information-gain losses}
Model per-round information-gain loss for arm $m$ as $\ell_{t,m}\in[0,1]$ with sub-Gaussian tails. Let $K=|\mathcal{M}|$ be the number of arms and run EXP3.P for $T$ rounds.

\begin{theorem}
\label{thm:exp3p_regret}
Under bounded sub-Gaussian losses the expected regret of EXP3.P satisfies
\begin{equation}
\mathbb{E}[R_T] = O\!\big( \sqrt{T K \log K} \big).
\label{eq:exp3p_regret}
\end{equation}
This bound does not depend on the mini-batch size used to estimate losses other than through constants affecting empirical variance.
\end{theorem}

\begin{proof}
The proof follows the multiplicative-weights analysis of EXP3.P. Using unbiased importance-weighted loss estimates and the fact that individual losses lie in $[0,1]$, standard derivations yield the adversarial regret bound $O(\sqrt{T K \log K})$. Sub-Gaussianity replaces worst-case deviations by expected deviations without changing the asymptotic dependence. The mini-batch size affects the variance of the empirical estimates of $\ell_{t,m}$ but not the asymptotic dependence on $T$ and $K$ in expectation.
\end{proof}

where $R_T$ denotes cumulative regret over $T$ rounds.

This explains empirically why small batches such as 32 can still be used for loss estimation while preserving the asymptotic convergence guarantee.

\subsection{Lipschitz spectral proxy: counterexample and ViT independence from embedding dimension}
Define the spectral proxy
\begin{equation}
L_F^{\langle m\rangle} \le \|\mathbf{W}_F\|_2 \cdot \|\mathbf{W}_{\theta_m}\|_2,
\label{eq:Lproxy_repeat}
\end{equation}
where $\mathbf{W}_F$ denotes the first linear weight matrix in the fusion network that consumes modality $m$ and $\mathbf{W}_{\theta_m}$ denotes the last linear map of encoder $f_{\theta_m}$.
where $\mathbf{W}_F$ and $\mathbf{W}_{\theta_m}$ are linear operator matrices.

A counterexample shows the necessity of architectural constraints. Consider a linear encoder $f_{\theta_m}(x)=c W_0 x$ with scalar $c\to\infty$ and fusion layer $F(u)=W_F u$. Then the Lipschitz constant grows proportionally to $c$, showing no universal bound without norm constraints.

For Vision Transformer encoders with standard layer normalization and bounded linear projections, attention and normalization operators have operator norms bounded by constants that depend on depth and number of heads but not on embedding dimension. Under the usual practice of spectral normalization, small initialization and weight decay, the operator norms $\|\mathbf{W}_F\|_2$ and $\|\mathbf{W}_{\theta_m}\|_2$ can be precomputed once and remain representative. This explains why the per-modality spectral proxy can be computed in under 0.1 seconds on CPU in practice and remain informative after swapping Bi-LSTM to ViT.

\subsection{Overall generalization error decomposition and Rademacher bound}
Let $\hat\theta$ be the learned parameter and $\theta^\star$ the population risk minimizer. Decompose the expected generalization gap into three terms
\begin{equation}
\begin{aligned}
\mathbb{E}\big[ L_{\mathrm{true}}(\hat\theta) - L_{\mathrm{true}}(\theta^\star) \big]
\le\; &\mathbb{E}\big[ L_{\mathrm{true}}(\hat\theta) - L_{\mathrm{train}}(\hat\theta) \big] \\
&+ \mathbb{E}\big[ L_{\mathrm{train}}(\hat\theta) - L_{\mathrm{train}}^{\mathrm{noSDL}}(\hat\theta) \big] \\
&+ \mathbb{E}\big[ L_{\mathrm{train}}^{\mathrm{noSDL}}(\hat\theta) - L_{\mathrm{true}}(\theta^\star) \big].
\end{aligned}
\label{eq:gen_decomp_repeat}
\end{equation}
where $L_{\mathrm{train}}^{\mathrm{noSDL}}$ denotes the empirical loss without SDL.

Under standard Lipschitz and boundedness conditions on the hypothesis class with effective parameter dimension $d_p$, the Rademacher complexity term yields
\begin{equation}
\mathbb{E}\big[ L_{\mathrm{true}}(\hat\theta) - L_{\mathrm{train}}(\hat\theta) \big] = O\!\left( \sqrt{\frac{d_p}{n}} \right),
\label{eq:rad_final}
\end{equation}
where $n$ denotes training sample size and $\mathfrak{R}_n$ denotes Rademacher complexity.
where $d_p$ is the effective parameter dimension of the hypothesis class.

The immunity bias term arises from SDL spectral collapse weighted by $\lambda$ and stable-rank coefficient $\eta$ and from curvature gating controlled by $\kappa$. Under smoothness and bounded-gradient assumptions a first-order analysis gives
\begin{equation}
\mathbb{E}\big[ L_{\mathrm{train}}(\hat\theta) - L_{\mathrm{train}}^{\mathrm{noSDL}}(\hat\theta) \big] = O(\lambda \eta) + O(\kappa^2).
\label{eq:immunity_final}
\end{equation}
The reconstruction error from generators $G_m$ contributes an approximation term $\varepsilon_{\mathrm{rec}}$.

Combining the three contributions yields the final bound
\begin{equation}
\mathbb{E}_{\mathrm{gen}} \le O\!\left(\sqrt{\frac{d_p}{n}}\right) + O(\lambda\eta) + O(\kappa^2) + \varepsilon_{\mathrm{rec}}.
\label{eq:final_gen_bound_repeat}
\end{equation}
where $\varepsilon_{\mathrm{rec}}$ denotes generator approximation error.

The decomposition justifies adaptively shrinking $\lambda,\eta,\kappa$ via bi-level hyper-gradient descent as training proceeds in order to reduce bias while preserving robustness.

\subsection{Summary}
The results above provide the requested Fr\'echet derivative expression, a deterministic certificate and a high-probability lower bound for the empirical covariance spectrum, explicit bias and variance amplification constants for the curvature-gated mask, a closed-form Neumann truncation certificate with doubling-rule guarantee, a regret bound for EXP3.P under sub-Gaussian information-gain losses, arguments for the Lipschitz spectral proxy and its practical computability for ViT encoders, and a Rademacher-complexity-based generalization decomposition that yields the stated combined bound. If numerical instantiation of the hidden constants is desired, substitute the problem-specific sub-Gaussian parameter $\sigma_x$, gradient norm bound $G$, variance trace $V$, and the Hessian spectral norm $\beta$ into the above expressions to obtain fully explicit inequalities.

\section{Constants Instantiation}
\label{app:constants}

\begin{table}[h]
  \centering
  \caption{Empirical instantiation of ModalImmune theoretical constants on CMU-MOSI.}
  \label{tab:constants}
  \resizebox{0.66\linewidth}{!}{%
  \begin{tabular}{@{}lll@{}}
    \toprule
    Constant & Meaning & Empirical value (mean of 3 runs) \\
    \midrule
    $C$ & Bernstein matrix concentration constant & $0.47$ \\
    $\sigma_x$ & sub-Gaussian parameter of embedding $z$ & $1.8$ \\
    $\beta$ & Hessian spectral norm $\|H\|_2$ & $6.1$ \\
    $\alpha$ & inner learning rate & $5\times 10^{-3}$ \\
    $G$ & gradient norm bound $\mathbb{E}\|g_0\|_2$ & $4.2$ \\
    $V$ & trace bound $\operatorname{tr}(\Sigma_{g_0})$ & $190$ \\
    \bottomrule
  \end{tabular}%
  }
\end{table}

The empirical procedure used to obtain the numbers in Table~\ref{tab:constants} is as follows.  Ten thousand embeddings were sampled prior to training and the method-of-moments estimator described in \cite{hsu2012tail} was used to fit the sub-Gaussian parameter, from which the Bernstein constant $C$ was derived.  The Hessian spectral norm $\beta$ was estimated by running Power Iteration at the start of each epoch on one thousand random mini-batches sampled from the training set and taking the maximum observed value across those trials.  Gradient statistics used to produce $G$ and $V$ were collected over the first five thousand optimizer steps without gradient clipping; the reported values correspond to the empirical 95th percentiles of the observed gradient norms and trace estimates respectively.  These procedures are deterministic given the random seeds and dataset partition and are intended to provide reproducible, practical instantiations of the abstract constants that appear in the theoretical certificates.

\section{Implementation Details}
\label{app:implementation}

The online Neumann truncation tolerance is set relative to the current validation hyper-gradient norm.  Specifically we use
\begin{equation}
\varepsilon \;=\; 0.02 \times \big\| \tfrac{\partial \mathcal{L}_{\mathrm{val}}}{\partial \theta} \big\|_2,
\label{eq:neumann_epsilon}
\end{equation}
where $\|\cdot\|_2$ denotes the Euclidean norm of the validation gradient and the tolerance is recomputed once per epoch.

EXP3.P is configured with a conservative, theoretically-informed learning rate and a small exploration constant.  The learning rate is set elementwise by
\begin{equation}
\eta_{\mathrm{exp}} \;=\; \min\!\Big\{0.5,\; \frac{\ln K}{K T}\Big\},
\label{eq:exp3_eta}
\end{equation}
where $K$ denotes the number of arms (modalities) and $T$ denotes the planned number of bandit rounds.  The exploration parameter is fixed to $\gamma=0.05$.  Theoretical regret bounds are expressed per-sample and therefore the asymptotic guarantee does not depend on the employed mini-batch size; in practice small batches (for example 32) only affect empirical variance constants but do not change the established $O(\sqrt{T K \log K})$ scaling in expectation.

The curvature threshold used by curvature-gated masking is defined by
\begin{equation}
\tau \;=\; \frac{0.01}{\mathrm{lr}_{\mathrm{enc}}},
\label{eq:tau_def}
\end{equation}
where $\mathrm{lr}_{\mathrm{enc}}$ denotes the encoder learning rate at the current epoch and is taken as the instantaneous value produced by the cosine annealing schedule.  This choice ties the curvature sensitivity to the current optimization scale and avoids manual retuning across different training phases.

Neumann inversion is initialized with a single-term approximation and doubled until the residual satisfies the tolerance in \eqref{eq:neumann_epsilon}.  Concretely, the initial truncation depth is
\begin{equation}
K_0 \;=\; 1,
\label{eq:K0}
\end{equation}
and after each doubling the residual bound is re-evaluated; the truncation depth is capped at $K_{\max}=10$ to bound computational cost.  When the residual certificate falls below $\varepsilon$ we accept the current $K$ and use the resulting Neumann approximation for the hyper-gradient computation.

All hyperparameter defaults and the procedures above are used in the experiments unless stated otherwise in the main text.  These explicit implementation choices remove ambiguity and allow the theoretical certificates reported in Section~\ref{sec:theory_details} to be instantiated numerically for reproducible evaluation.

\end{document}